\documentclass[runningheads]{llncs}

\usepackage[colorlinks,
           linkcolor=red,
           filecolor=green, 
           urlcolor=blue,
	       citecolor=green,
           ]{hyperref}
\usepackage{graphicx}
\usepackage{bm}
\usepackage{tikz}
\usepackage{comment}

\usepackage{amsmath,amssymb} 
\usepackage{color}

\usepackage[accsupp]{axessibility}  

\usepackage{amssymb,mathrsfs,amsmath}
\usepackage{microtype}
\usepackage{graphicx}
\usepackage{subfigure}
\usepackage{booktabs} 
\usepackage{lipsum}

\usepackage{bbm}
\usepackage{multirow}
\usepackage{wrapfig}
\usepackage{amsthm}
\usepackage{color,xcolor}
\usepackage{bbding}
\usepackage{threeparttable}
\usepackage[ruled]{algorithm2e}
\newcommand{\eg}{{\em e.g.}}           
\newcommand{\ie}{{\em i.e.}}           
\newcommand{\etc}{{\em etc.}}         

\definecolor{light-gray}{gray}{0.6}
\definecolor{light-gray-2}{gray}{0.9}

\begin{document}
\pagestyle{headings}
\mainmatter
\def\ECCVSubNumber{3058}  

\title{RDA: Reciprocal Distribution Alignment for Robust Semi-supervised Learning} 

\titlerunning{RDA: Reciprocal Distribution Alignment for Robust SSL}
%
\author{Yue~Duan\inst{1} \and
Lei~Qi\inst{2} \and
Lei~Wang \inst{3} \and
Luping~Zhou\inst{4} \and
Yinghuan~Shi\inst{1}$^\star$
}
\authorrunning{Y. Duan et al.}
%
\institute{$^1$ Nanjing University, China \quad $^2$ Southeast University, China \\
$^3$ University of Wollongong, Australia \quad $^4$  University of Sydney, Australia
}
\maketitle
\newcommand\blfootnote[1]{%
\begingroup
\renewcommand\thefootnote{}\footnote{#1}%
\addtocounter{footnote}{-1}%
\endgroup
}
\blfootnote{$^\star$ Corresponding author: Y. Shi (e-mail: syh@nju.edu.cn).}
\blfootnote{Y. Duan, Y. Shi are with the National Key
Laboratory for Novel Software Technology and the National Institute of
Healthcare Data Science, Nanjing University.}

\begin{abstract}
In this work, we propose Reciprocal Distribution Alignment (RDA) to address semi-supervised learning (SSL), which is a hyperparameter-free framework that is independent of confidence threshold and works with both the matched (conventionally) and the mismatched class distributions. Distribution mismatch is an often overlooked but more general SSL scenario where the labeled  and the unlabeled data do not fall into the identical class distribution. This may lead to the model not exploiting the labeled data reliably and drastically degrade the performance of SSL methods,  which could not be rescued by the traditional distribution alignment. In RDA,  we enforce a reciprocal alignment on the distributions of the predictions from two classifiers predicting pseudo-labels and complementary labels on the unlabeled data. These two distributions, carrying complementary information, could be utilized to regularize each other without any prior of class distribution. Moreover, we theoretically show that RDA maximizes the input-output mutual information. Our approach achieves promising performance in SSL under a variety of scenarios of mismatched distributions, as well as the conventional matched SSL setting. Our code is available at: \url{https://github.com/NJUyued/RDA4RobustSSL}.
\keywords{distribution alignment, mismatched distributions}
\end{abstract}

\section{Introduction}
\label{sec:intro}
Semi-supervised learning (SSL) leverages the abundant unlabeled data to alleviate the lack of labeled data for machine learning~\cite{chapelle2009semi,zhu2017semi,van2020survey}. 
Lately, \textit{confidence-based pseudo-labeling}~\cite{sohn2020fixmatch,li2021comatch} and \textit{distribution alignment}~\cite{bridle1992unsupervised,berthelot2020remixmatch,li2021comatch,gong2021alphamatch} have been introduced to SSL, boosting the performance to a new height. These techniques improve the label imputation for unlabeled data, which alleviates the confirmation bias~\cite{arazo2020pseudo}. 
In brief, pseudo-labeling aims to achieve entropy minimization \cite{grandvalet2005semi} by producing hard labels. Recently, FixMatch~\cite{sohn2020fixmatch} utilizes the confidence-based threshold to select more accurate pseudo-labels 
and proves the superiority of this technique.
Despite this threshold preventing the model from risk of noisy pseudo-labels, since the learning difficulties of different classes are different, a fixed threshold is not a ``silver bullet'' for all scenarios of SSL.
Although \cite{xu2021dash,zhang2021flexmatch} demonstrate the potential to dynamically adjust the threshold, the adjustment is complicated and the waste of unlabeled data with low confidence will become a latent limitation \cite{duan2022mutexmatch}.
We try to ask --- \textit{is the confidence-based threshold really necessary for pseudo-labeling?} 

\newcommand{\mzz}{2.9cm}
\newcommand{\widd}{4cm}
\newcommand{\mysize}{5cm}
\begin{figure}[t]
  \centering
  \subfigure[]{
  \includegraphics[width=\mzz]{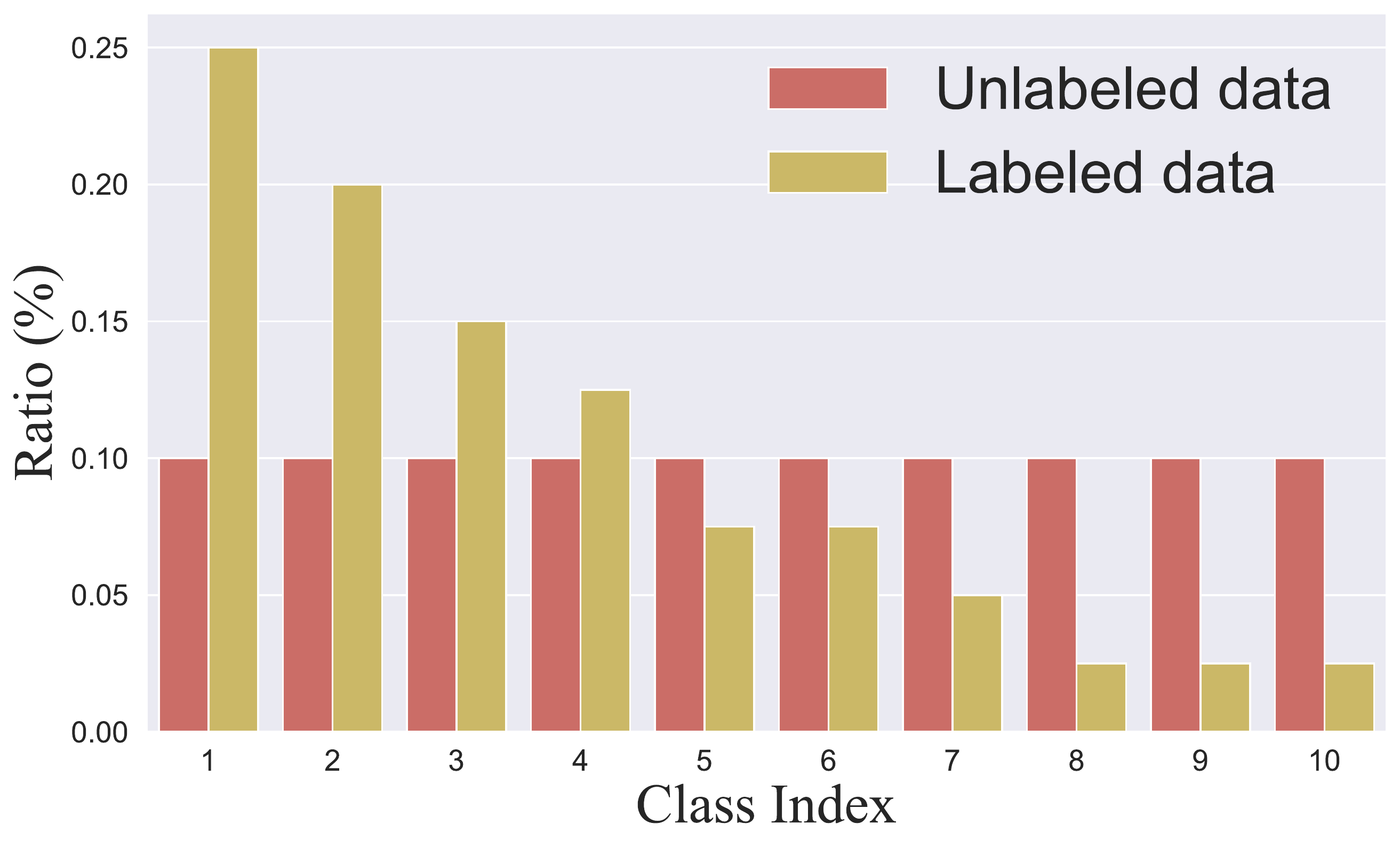}
  \label{fig:mis1}
  }
  \hspace{-4mm}
  \subfigure[]{
  \includegraphics[width=\mzz]{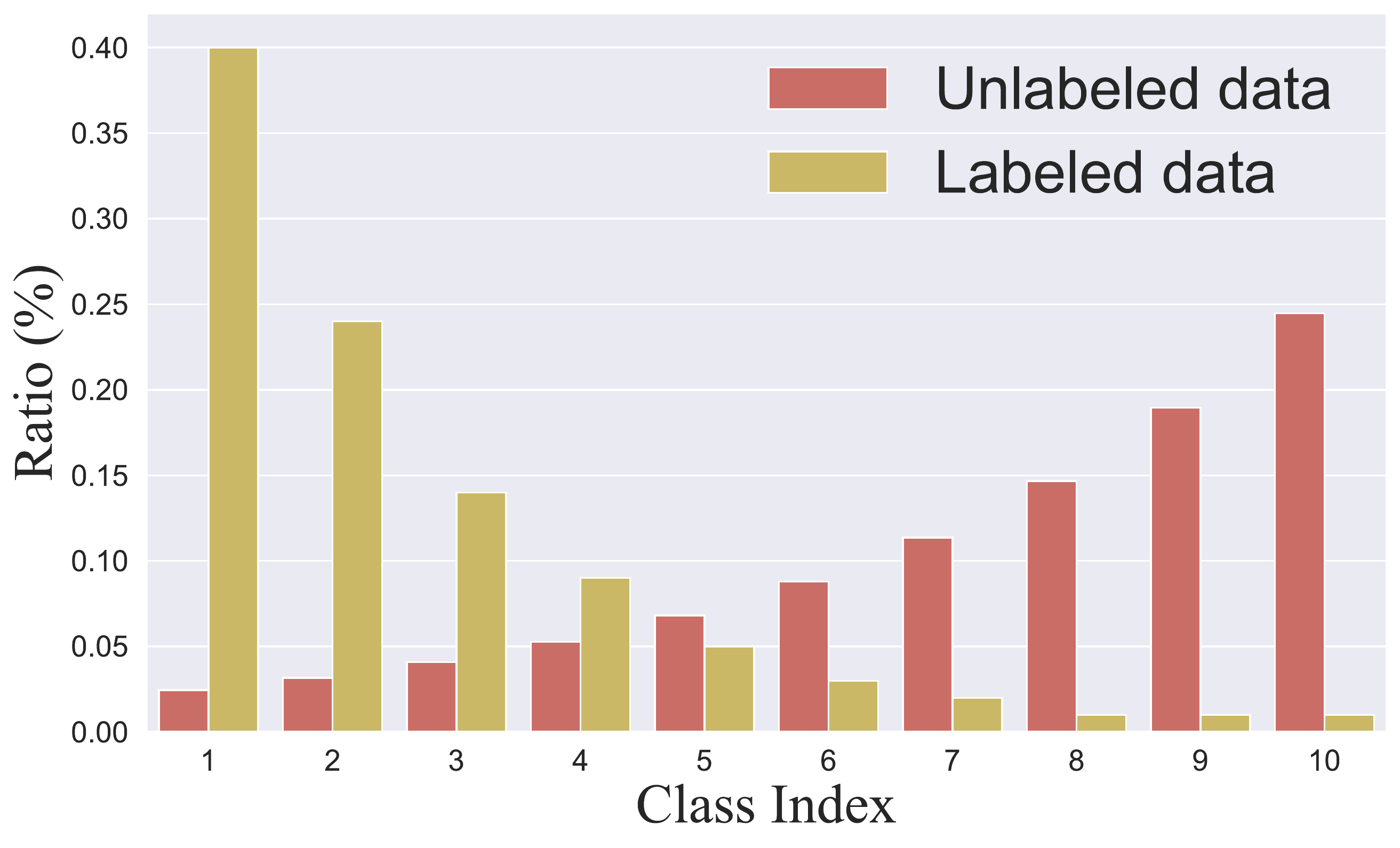}
  \label{fig:mis2}
  }
  \hspace{-4mm}
  \subfigure[]{
  \includegraphics[width=\mzz]{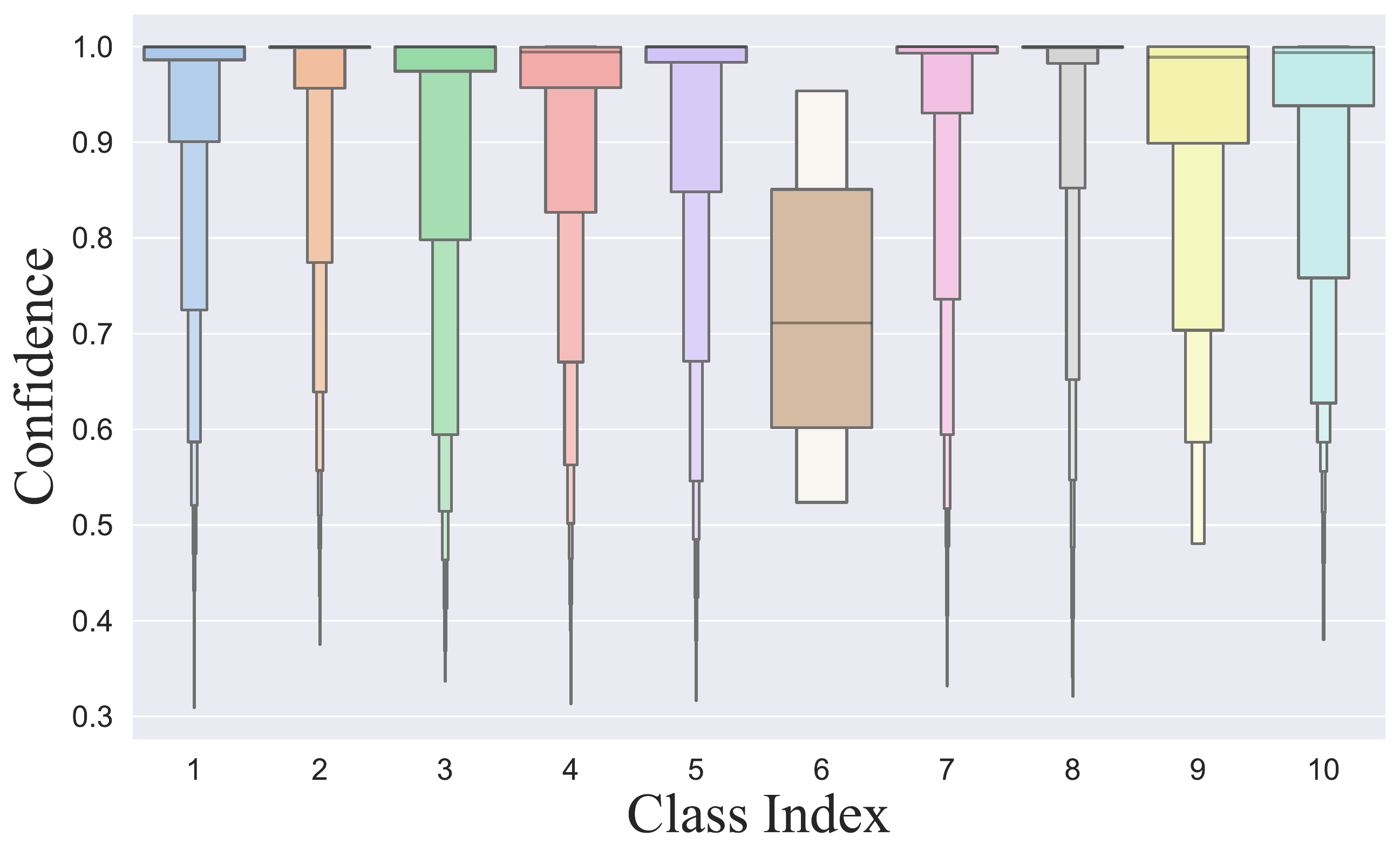}
  \label{fig:mis3}
  }
  \hspace{-4mm}
  \subfigure[]{
  \includegraphics[width=\mzz]{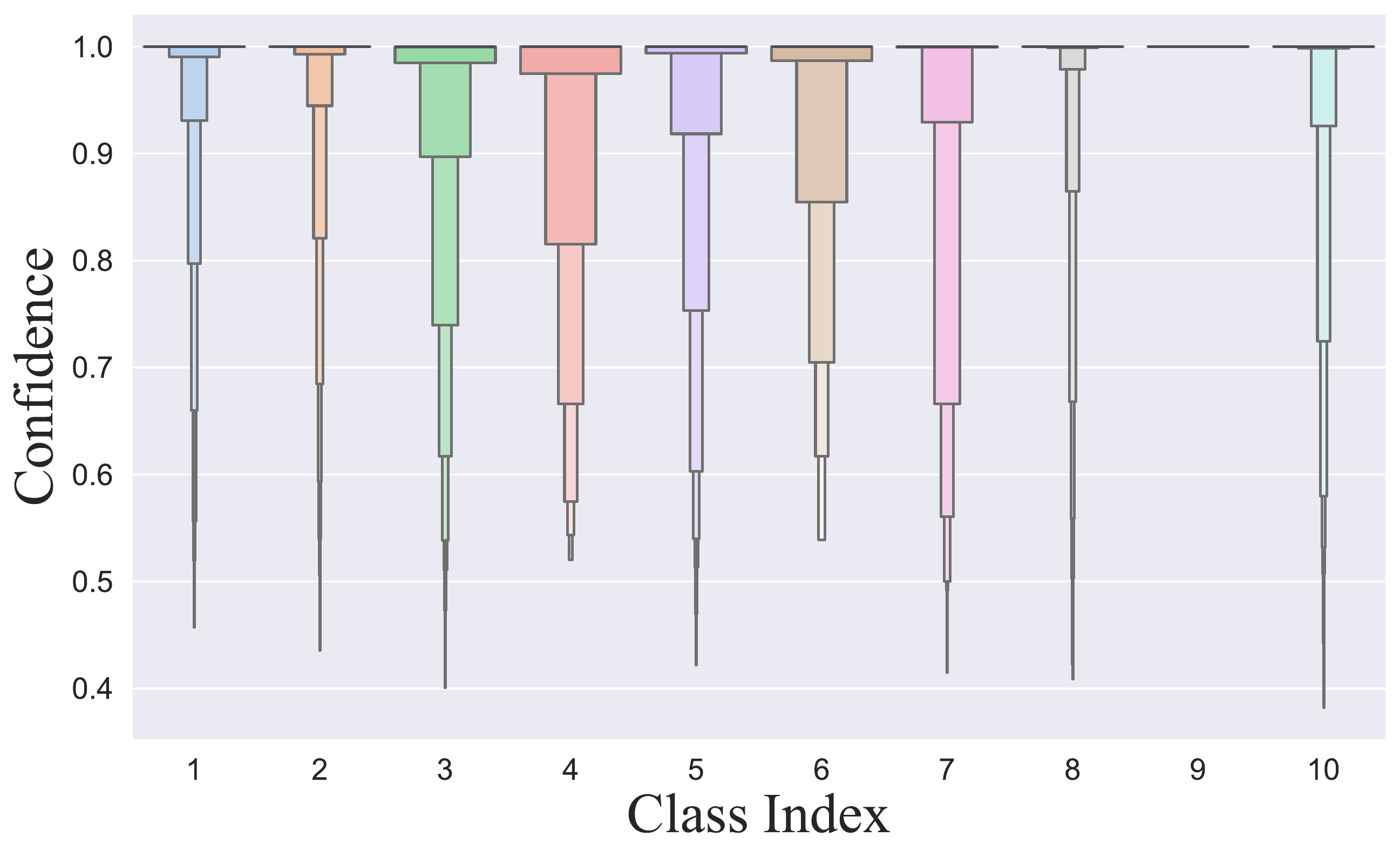}
  \label{fig:mis4}
  }
  \caption{
  Some examples of mismatched distributions in SSL. 
  The x-axis represents the index of classes in CIFAR-10. In (a) and (b), the figures show the distributions of the labeled and unlabeled data.
  In (c) and (d), the figures show the confidences of FixMatch's predictions on the unlabeled data. Letter-value plots  \cite{letter-value-plot} are displayed for multi-level quantile information. In (a) and (c), we show imbalanced labeled data  and  balanced unlabeled data  with 40 labels $N_{0}=10$. In (b) and (d), the labeled and unlabeled data are mismatched and imbalanced with 100 labels, $N_{0}=40$ and $\gamma=10$.  More details about imbalance ratio $N_{0}$ and $\gamma$ can be found in Sec. \ref{sec:cxcu}. In (c) and (d), we can see that the confidences of FixMatch's predictions on the unlabeled data of different classes are totally irregular, which means it is difficult for us to adjust the confidence threshold to judge whether the prediction is correct. 
  \ie, confidence-based pseudo-labeling is also not suitable for the mismatched distributions.
  }
  \label{fig:mis}
\end{figure}
Motivated by this, we rethink pseudo-labeling in a hyperparameter-free way while 
noticing that distribution alignment (DA) has been introduced to SSL \cite{berthelot2020remixmatch,li2021comatch,gong2021alphamatch}. DA scales the predictions on unlabeled data by prior information about labeled data distribution for strong regularization on the pseudo-labels, which can mitigate the confirmation bias. 
Inspired by this, we consider only using DA to improve the pseudo-labels without additional hyperparameters, \ie, DA is enough for pseudo-labeling.
Meanwhile, DA shows great potential in addressing the SSL under long-tailed distribution \cite{wei2021crest}. 
We expect that this technique can play a positive role in SSL in a more general scope.
However, even though DA could help us improve pseudo-labeling by protecting SSL from noise, it is based on a strong assumption: \textit{``labeled data and unlabeled data share the same distribution,''} 
\eg, they are all balanced in CIFAR-10. The scenarios of \textit{mismatched distributions} have not been widely discussed,
\ie, the distribution of labeled data doesn't match that of unlabeled data, 
which is illustrated in Fig. \ref{fig:mis}. Some typical scenarios lead to mismatched distributions, such as biased sampling, label missing not at random \cite{hernan2010causal} and so on. Mismatched distributions might cause biased pseudo-labels, significantly degrading the SSL model performance which is demonstrated by experimental results in Sec. \ref{sec:res}.
Under mismatched distributions, we cannot simply use the distribution of the labeled data to align predictions on unlabeled data with a very different distribution. 
This drives us to explore a more general distribution alignment to meet the above challenge of mismatched distributions.

\begin{figure}[t]  
    \setlength{\abovecaptionskip}{-1em}
  \begin{center}
  \resizebox{\linewidth}{!}{          
      \includegraphics{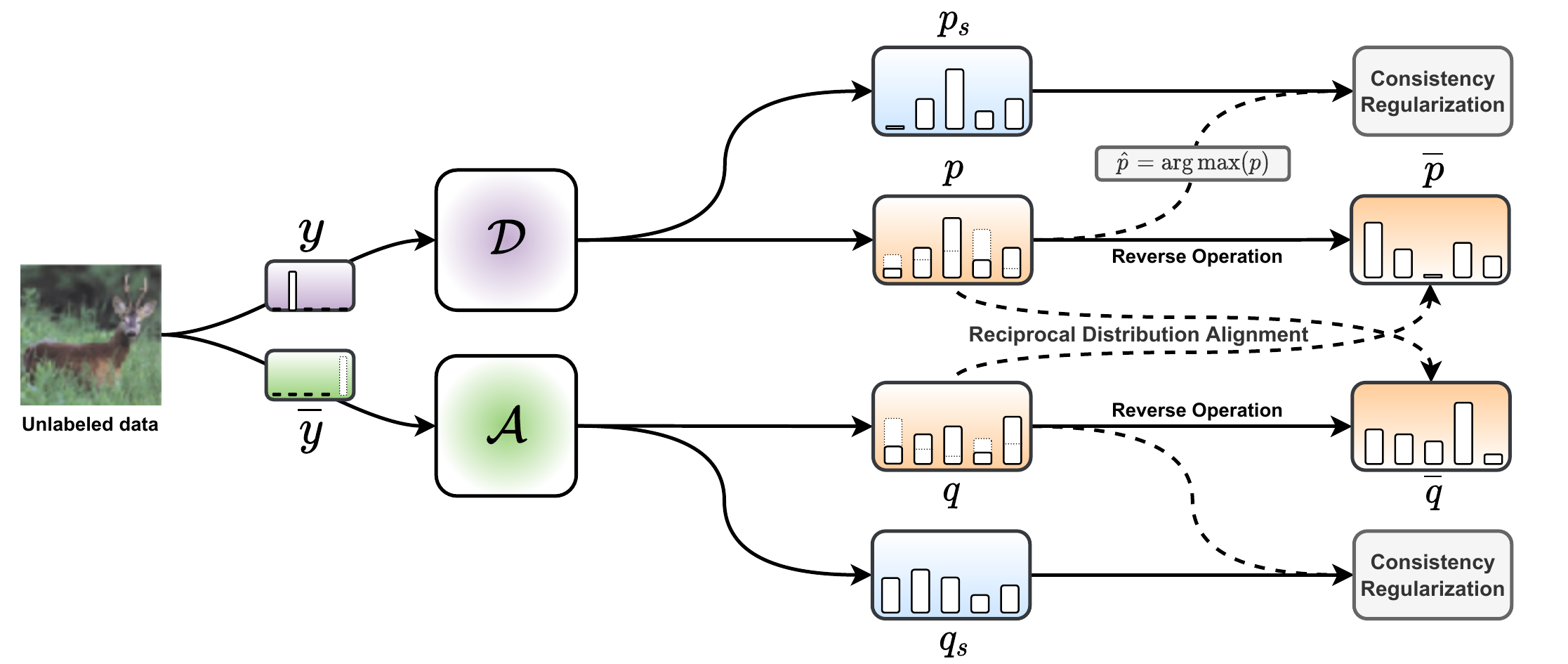}
  }  
  \end{center}
  \caption{Diagram of proposed Reciprocal Distribution Alignment (RDA). We use ground-truth label $y$ and complementary label $\overline{y}$ (dash line means $\overline{y}$ is selected randomly from classes excluding ground-truth label) of labeled data to train Default Classifier $\mathcal{D}$ and Auxiliary Classifier $\mathcal{A}$, respectively. Given an unlabeled sample $u$, $\mathcal{D}$ predicts pseudo-label $p$ and $\mathcal{A}$ predicts complementary label $q$ for its weakly-augmented version. RDA is applied on $p$ and $q$ by reciprocally scaling each other to the distributions of their reversed versions obtained by \textit{Reverse Operation} (Proposition \ref{pro:1}). We then enforce consistency regularization on the aligned pseudo-label and complementary label against corresponding predictions for strongly-augmented $u$, \ie, $p_{s}$ (from $\mathcal{D}$) and $q_{s}$ (from $\mathcal{A}$).}
  \label{fig1}
\end{figure}

Given motivations mentioned above, we propose Reciprocal Distribution Alignment (\textbf{RDA}) to establish a promising semi-supervised learning paradigm, which provides an integrated scheme to handle both the matched and mismatched scenarios
in SSL. To relax the assumption about the class distribution of unlabeled data, we consider starting from the model itself to tap
the potential guidance of distribution by regularizing the predictions from
complementary perspectives. Inspired by \cite{ishida2018complementary,kim2019nlnl,rizve2021in}, we consider simultaneously predict the class labels and their complementary labels (\ie, indicating what class a sample is not), and utilize their distributions to regularize each other.
Thus, we introduce two classifiers to RDA, one is Default Classifier (\textbf{DC}) and the other is  Auxiliary Classifier (\textbf{AC}). Specifically, DC and AC are used to predict pseudo-labels and complementary labels for unlabeled data, respectively. The
pseudo-labels and the complementary labels could be transformed into each other through their reversed version using the \textit{Reverse Operation} defined in Proposition \ref{pro:1} in Sec. \ref{sec:bda}. Then a reciprocal alignment is employed to adjust the distributions of DC’s predictions and AC’s predictions by scaling them according to their corresponding reversed versions. We prove that RDA produces a ``high-entropy'' form of prediction distribution, which lead to maximizing the objective of input-output mutual information \cite{bridle1992unsupervised,berthelot2020remixmatch}.
With the aligned pseudo-labels and complementary labels, the commonly used consistency regularization is further applied on on DC and AC, respectively, which helps the model remain unchanged prediction on perturbed data.
RDA could be applied to help the model improve pseudo-labels without suffering from the threat of mismatched distributions since no prior information about class distribution of data is used.
A diagram of RDA is shown in Fig. \ref{fig1}.

Despite its simplicity, our method shows superior performance in various settings, \eg, on widely-used SSL benchmark CIFAR-10, RDA achieves an accuracy of 92.03$\pm$2.01\% with only 20 labels in the conventional setting, and in mismatched distributions, outperforms CoMatch \cite{li2021comatch}, a recently-proposed algorithm for SSL, by up to a 52.09\% gain on accuracy. Besides the significant performance improvement, our contributions can be presented as follows:
 \begin{itemize}
 \item[$\bullet$]
  We propose Reciprocal Distribution Alignment (RDA), a novel SSL algorithm, which can improve pseudo-labels in a hyperparameter-free way.
 \item[$\bullet$]
  RDA can be safely applied to SSL in both the conventional setting and the scenarios of mismatched distributions.
 \item[$\bullet$]
  We theoretical show that RDA could optimize the objective of mutual information between input data and predictions \cite{bridle1992unsupervised,berthelot2020remixmatch} under the premise of rational use of class distribution guidance information. 
 \end{itemize}

\section{Related Work}
\label{sec:rework}

\noindent \textbf{Pseudo-labeling Based Entropy Minimization.}
Entropy minimization is a significant idea in recent SSL methods, which is closely related to pseudo-labeling (\ie, convert model's predictions to hard labels to  reduce entropy)~\cite{lee2013pseudo,sohn2020fixmatch,li2021comatch,yang2021mining}. In another word, pseudo-labeling results in a form of entropy minimization~\cite{grandvalet2005semi}. This idea argues that model should ensure classes are well-separated while utilizing unlabeled data, which can be achieved by
encouraging the model output prediction with low entropy~\cite{grandvalet2005semi}. 
Recent SSL algorithms like~\cite{sohn2020fixmatch,li2021comatch,xu2021dash,zhao2022lassl} set a confidence-based threshold to refine the pseudo-labels and obtain outstanding performance. 
However, the existence of confidence threshold leads to a  waste of unlabeled samples with low confidence because they were filtered out.
Moreover, it will lead to a significant increase in the cost of dynamic adjustment on confidence threshold like~\cite{xu2021dash,zhang2021flexmatch}.
Meanwhile, under mismatched distributions, it is not reasonable to use a fixed threshold for all classes to filter pseudo-labels, because the model will also be affected by the unlabeled data with a potential risk of unknown distribution. 
In this work, we use distribution alignment to improve pseudo-labeling in a hyperparameter-free way which can achieve a better performance than algorithms introducing  confidence threshold.

\noindent\textbf{Distribution Alignment in SSL.}
Distribution alignment is proposed in  \cite{bridle1992unsupervised} and originally applied to SSL in \cite{berthelot2020remixmatch}.
Briefly,  \cite{berthelot2020remixmatch} integrates it into pseudo-label inference step without additional loss terms or hyper-parameters. The main idea is the marginal distribution of predictions on unlabeled data and the marginal distribution of ground-truth labels should be consistent. This alleviates the confirmation bias~\cite{arazo2020pseudo}  by improving pseudo-labels with the help of distributional guidance information. For class-imbalanced semi-supervised learning, \cite{wei2021crest} improves this technique by replacing the distribution of ground-truth labels with a smoothed form, resulting in superior performance in this setting. This improved distribution alignment in \cite{wei2021crest} helps the model benefit from rebalancing distribution.

In short, the objective of distribution alignment is to maximize the mutual information between the predictions and input data, \ie, input-output mutual information \cite{bridle1992unsupervised,berthelot2020remixmatch}. Denoting the input data as $x$, the class prediction for $x$ as $y$, and the predicted class distribution as $P(y|x)$, we can formalize this objective as:
\begin{equation}
    \mathcal{I}(y;x)=\mathcal{H}(\mathbbm{E}_{x}[P(y|x)])-\mathbbm{E}_{x}[{\mathcal{H}(P(y|x))}],
    \label{eq:remix}
\end{equation}
where $\mathcal{H}(\cdot)$ refers to the entropy. For specific, distribution alignment aims to maximize the term $\mathcal{H}(\mathbbm{E}_{x}[P(y|x)])$.
However, the implementation of this technique in both \cite{berthelot2020remixmatch} and \cite{wei2021crest} is based on an idealized assumption: \textit{``labeled and unlabeled data fall in the same distribution.''} More realistically, we cannot guarantee that the distribution of labeled data matches that of unlabeled data. Such mismatched distributions can cause the distribution alignment in \cite{berthelot2020remixmatch,wei2021crest} to fail and is even detrimental to the model's predictions on unlabeled data. In this work, we propose Reciprocal Distribution Alignment without the assumption of matched distributions and any prior information about the labeled data distribution.

\section{Method}
In this section, we discuss the setting of mismatched distributions in SSL and propose a novel SSL algorithm called Reciprocal Distribution Alignment (\textbf{RDA}) without additional hyper-parameters to improve pseudo-labeling in various scenarios of SSL. Moreover, we theoretically analyze the effectiveness of our method. 
\subsection{Matched and Mismatched Distributions in SSL}
\label{sec:md}
In semi-supervised learning, we have a training set divided into labeled portion $\mathcal{X}$ and unlabeled portion $\mathcal{U}$. We denote class distribution of $\mathcal{X}$ as $\mathcal{C}_{x}$ and class distribution of $\mathcal{U}$ as $\mathcal{C}_{u}$. Note that $\mathcal{C}_{u}$ is inaccessible in training. 
Given $x\in\mathcal{X}$ with corresponding label $y$ and unlabeled data $u\in\mathcal{U}$, we can review the SSL algorithms as the following optimization task:
\begin{equation}
\min \mathcal{L} =\mathcal{L} _{sup}(x,y;\theta)+\mathcal{L}_{unsup}(u;\theta),
\end{equation}
where $\theta$ is the parameters of the model, $\mathcal{L} _{sup}$ is supervised loss for the labeled data and $\mathcal{L} _{unsup}$ is unsupervised loss for the unlabeled data. 
Recent pseudo-labeling based SSL methods try to impute the unknown label of $u$ for $\mathcal{L} _{unsup}$. Therefore, the accuracy of pseudo-labels has become the top priority.
In the traditional SSL setting, we assume $\mathcal{C}_{x}\approx\mathcal{C}_{u}$. Under this assumption, we can use $\mathcal{C}_{x}$ to guide the prediction for $u$ by distribution alignment~\cite{berthelot2020remixmatch,li2021comatch}, which can improve the performance of consistentency-based or pseudo-labeling based methods \cite{berthelot2020remixmatch,sohn2020fixmatch,li2021comatch,gong2021alphamatch}. Unfortunately, this assumption is too impractical and idealistic. More in line with the actual situation is $\mathcal{C}_{x}\not\approx\mathcal{C}_{u}$, which is called \textit{mismatched distributions} in SSL.
Unlike the conventional SSL, in mismatched distributions, 
the model learns a distribution from $\mathcal{C}_{x}$ that differs from $\mathcal{C}_{u}$, so it cannot correctly predict the pseudo-labels. In other words, the distribution gap caused by mismatch leads to strong confirmation bias \cite{arazo2020pseudo}, which could affect the performance of the model. 
It is worth noting that the distribution alignment used in \cite{wei2021crest} to solve the SSL under long-tail distribution also cannot be applied to the mismatched scenarios because \cite{wei2021crest} still depends on the assumption of matched distributions.
To design a method that can tackle mismatched scenarios in SSL, we must face to $\mathcal{C}_{x}\not\approx\mathcal{C}_{u}$, and abandon prior of $\mathcal{C}_{x}$ used in previous method~\cite{berthelot2020remixmatch,wei2021crest}.

\subsection{Overview}
We introduce two classifiers for our method. One is called Default Classifier (DC) $\mathcal{D}$ and the other is called Auxiliary Classifier (AC) $\mathcal{A}$. In a nut shell, for an  unlabeled image, $\mathcal{D}$ is used to predict pseudo-label  and $\mathcal{A}$ is used to predict complementary label. We obtain labeled data $\mathcal{X} =\{(x_{b},y_{b})\}^{B}_{b=1}$ consisting of $B$ images and unlabeled data $\mathcal{U}=\{(u_{b})\}^{\mu B}_{b=1}$ consisting of $\mu B$ images in a batch of data. At first, we construct complementary label $\overline{y}$ for every labeled data by their ground-truth. 
Complementary label~\cite{ishida2017learning,ishida2018complementary} represents which class the sample does not belong to. Denoting $y\in \mathcal{Y} =\{1,\dots,n\} $ as the ground-truth label of $x$ where $n$ is the number of classes, following \cite{kim2019nlnl}, the complementary label of $x$ is randomly selected from $\mathcal{Y} \setminus \{y\}$, which is denoted as $\overline{y}$.

Following~\cite{sohn2020fixmatch}, we integrate \textit{consistency regularization} into RDA. Weak and strong augmentations are performed on images then we enforce consistency regularization on $\mathcal{D}$ and $\mathcal{A}$.
Denoting $u_{w}$ as the weakly-augmented image and $u_{s}$ as the strongly-augmented image for the same unlabeled data $u$, let $y_{c}$ be the class prediction for input image. 
$P_{G}(y_{c}|\cdot)$ refers to the predicted class distribution outputted by classifier $G$ for input. We can obtain pseud-labels $p=P_{\mathcal{D}}(y_{c}|u_{w})$, $p_{s}=P_{\mathcal{D}}(y_{c}|u_{s})$, and complementary labels $q=P_{\mathcal{A}}(y_{c}|u_{w})$, $q_{s}=P_{\mathcal{A}}(y_{c}|u_{s})$ respectively. 
Note that  $p$, $q$ are $n$-dimensional vectors of class probability where $n$ is the number of classes. $p_{i}$, $q_{i}$ represent the probability of belonging to the $i$-th class in the predictions.
Then, dual consistency regularization can be achieved by minimizing the default consistency loss $\mathcal{L}_{cd}$ and auxiliary consistency loss $\mathcal{L}_{ca}$:
\begin{equation}
\mathcal{L}_{cd}=\frac{1}{\mu B}\sum_{n = 1}^{\mu B}H(\hat{p}_{n},p_{s,n}),  
\end{equation}
\begin{equation}
\mathcal{L}_{ca}=\frac{1}{\mu B}\sum_{n = 1}^{\mu B}H(q_{n},q_{s,n}),  
\label{eq:cr} 
\end{equation}
where $H(\cdot,\cdot)$ refers to the cross-entropy loss and $\hat{p}=\arg\max (p)$, which means we use hard labels for consistency regularization on $\mathcal{D}$. Differently, soft labels are used for $\mathcal{A}$ instead.
RDA exploits all unlabeled data for training, whereas previous consistency-based methods waste low-confidence data \cite{sohn2020fixmatch,li2021comatch,xu2021dash}. 

In addition, we enforce cross-entropy loss on  $\mathcal{D}$ between weakly-augmented version of $x$ (denoted as $x_{w}$) and $y$, and on $\mathcal{A}$ between $x_{w}$ and $\overline{y}$ respectively:
\begin{equation}
\mathcal{L}_{sd}=\frac{1}{B}\sum_{n = 1}^{B}H(y_{n},P_{\mathcal{D}}(y_{c}|x_{w,n})),
\end{equation}
\begin{equation}
\label{eq:sup} 
\mathcal{L}_{sa}=\frac{1}{B}\sum_{n = 1}^{B}H(\overline{y}_{n},P_{\mathcal{A}}(y_{c}|x_{w,n})),
\end{equation}
where $\mathcal{L}_{sd}$ is default supervised loss for $\mathcal{D}$ and $\mathcal{L}_{sa}$ is auxiliary supervised loss for $\mathcal{A}$. To sum up, RDA jointly optimizes four losses mentioned above:
\begin{equation}
\mathcal{L} =\mathcal{L}_{sd}+\lambda_{a} \mathcal{L}_{sa}+\lambda_{cd}\mathcal{L}_{cd}+\lambda_{ca} \mathcal{L}_{ca},
\end{equation}
where $\lambda_{a}$, $\lambda_{cd}$ and $\lambda _{ca}$ are trade-off coefficients and are all set to $1$ for simplicity.

Previous entropy minimization based methods like~\cite{sohn2020fixmatch,li2021comatch,xu2021dash} achieve superior performance in SSL by pseudo-labeling. Their key to success is the confidence threshold set to control the selection of pseudo-labels. To eliminate this hyper-parameter that becomes cumbersome in mismatched distributions, we consider a way to improve pseudo-labels using only distribution alignment. 
According to Eq. \eqref{eq:remix}, we can formalize the objective of distribution alignment for $\mathcal{D}$ as:
\begin{equation}
    \mathop{\max}_{\mathcal{D}}\mathcal{H}[\mathbbm{E}_{u}(P_{\mathcal{D}}(y_{c}|u_{w}))],
    \label{eq:h1}
\end{equation}
where $\mathcal{H}(\cdot)$ refers to the entropy. Likewise, we formalize the objective of distribution alignment for $\mathcal{A}$ as:
\begin{equation}
    \mathop{\max}_{\mathcal{A}}\mathcal{H}[\mathbbm{E}_{u}(P_{\mathcal{A}}(y_{c}|u_{w}))].
    \label{eq:h2}
\end{equation}
This two objectives encourage model to make predictions with equal frequency but these are not necessarily useful when dataset's class distribution of ground-truth is not uniform. We use Reciprocal Distribution Alignment descried in next paragraph to incorporate  these two objectives.

\subsection{Reciprocal Distribution Alignment}
\label{sec:bda}
Following \cite{berthelot2020remixmatch}, we notice that making one distribution approach to another (distribution of labeled data is used in \cite{berthelot2020remixmatch}) can achieve the purpose of maximizing Eq. \eqref{eq:remix}. In this way, a form of ``high entropy'' could be achieved for the objective described by Eqs. \eqref{eq:h1} and \eqref{eq:h2}. In brief, we define the objective over $\mathcal{D}$ and $\mathcal{A}$ as:
\begin{equation}
    \mathop{\max}_{\mathcal{D},\mathcal{A}}h(\mathcal{D},\mathcal{A})=\mathcal{H}[\mathbbm{E}_{u}(p)]+\mathcal{H}[\mathbbm{E}_{u}(q)].
    \label{eq:obj}
\end{equation}
However, due to the existence of mismatched scenarios, the class distribution of labeled data cannot be directly used for alignment like \cite{berthelot2020remixmatch}. So, next we will use the distribution of  class predictions (\ie, $\mathbbm{E}_{u}(p)$) and the distribution of  complementary class predictions (\ie, $\mathbbm{E}_{u}(q)$) to build a reciprocal alignment.
Considering there is no strong correlation between the distribution of class predictions and that of complementary class predictions, we assume that $\mathcal{A}$ is used to predict pseudo-label $\overline{q}$ (a ``reversed'' version of $q$), so that the ``reversed'' version of $\mathbbm{E}_{u}(q)$ (\ie, $\mathbbm{E}_{u}(\overline{q})$) can be used to align $\mathbbm{E}_{u}(p)$.
\begin{proposition}[Reverse Operation]
In the case of using $\mathcal{A}$ to predict pseudo-labels, we have $\overline{q}=\textrm{Norm}(\mathbbm{1}-q)$, where $\mathbbm{1}$ is all-one vector and $\textrm{Norm}(x)$ is the normalized operation defined as
$x'_{i}=x_{i}/\sum_{j=1}^{n}x_{j}$, $i\in (1,\dots,n).$
\label{pro:1}
\end{proposition}
\begin{proof}
Assuming we use $\mathcal{A}$ to predict pseudo-label $\overline{q}$, ideally, the probability of one class (\ie, $q_{i}$) should randomly fall on a class which is different from the class predicted currently (\ie, $\overline{q}_{j}$ where $j\neq i$). Thus, for any $\overline{q}_{j}\in\overline{q}$, its value is the sum of the values randomly assigned to it by all $q_{i}$:
\begin{align}
\overline{q}_{j}&=\sum_{i=1,i\neq j}^{n}\frac{q_{i}}{n-1}= \frac{1-q_{j}}{n-1} \label{eq:eqs1}.
\end{align}
Rewriting it we obtain:
\begin{align}
  \overline{q}_{j}
    & =\frac{1-q_{j}}{n-\sum_{k=1}^{n}q_{k}}\
      =\frac{1-q_{j}}{(1-q_{1})+\dots+(1-q_{n})}\nonumber\\
      & =\frac{1-q_{j}}{\sum_{k=1}^{n}(1-q_{k})}
      = \textrm{Norm}(1-q_{j}).
\end{align}
Now, $\overline{q}=\textrm{Norm}(1-q)$ follows by combining the similar proof for any $q_{i}\in q$.
\end{proof}

Likewise, if we use $\mathcal{D}$ to predict complementary label $\overline{p}$, it can be calculated as $\overline{p}=\textrm{Norm}(\mathbbm{1}-p)$. 
By Eq. \eqref{eq:eqs1}, we notice that \textit{Reverse Operation} does not change the relative relationship between classes in the class distribution, but just reverses the order, which allows us to still obtain helpful guidance information from the pseud-label and complementary label perspectives.

Then, distribution alignment is conducted on $\mathbbm{E}_{u}(p)$ by scaling it to $\mathbbm{E}_{u}(\overline{q})$. Reciprocally, we align $\mathbbm{E}_{u}(q)$ by scaling it to $\mathbbm{E}_{u}(\overline{p})$. Following~\cite{berthelot2020remixmatch}, we also integrate  distribution alignment into RDA without hyper-parameters.
We compute the moving average $\Psi(\cdot)$ of $p$, $q$, and their reversed version $\overline{p}$, $\overline{q}$ over last 128 batches, which can respectively serve as the estimation of $\mathbbm{E}_{u}(p)$, $\mathbbm{E}_{u}(q)$, $\mathbbm{E}_{u}(\overline{p})$ and $\mathbbm{E}_{u}(\overline{q})$.  Given an unlabeled image $u$, we scale the prediction of $\mathcal{D}$, \ie, pseudo-label $p$ by:
\begin{equation}
\Tilde{p}=\textrm{Norm}(p\times\frac{\Psi(\overline{q})}{\Psi(p)}),
\label{eq:da1}
\end{equation}
where $\Tilde{p}$ is an aligned probability distribution. Then, $\hat{\Tilde{p}}=\arg\max\Tilde{p}$ is used as hard pseudo-label for default consistency loss $\mathcal{L}_{cd}$.
Meanwhile, we scale the prediction of $\mathcal{A}$, \ie, complementary label $q$ by:
\begin{equation}
\Tilde{q}=\textrm{Norm}(q\times\frac{\Psi(\overline{p})}{\Psi(q)}),
\label{eq:da2}
\end{equation}
where $\Tilde{q}$ is an aligned probability distribution. Then $\Tilde{q}$ is used as soft complementary label for auxiliary consistency loss $\mathcal{L}_{ca}$. The following theorem shows why RDA results in maximizing the objective Eq. \eqref{eq:obj}. In this way, the input-output mutual information could be maximized, boosting the model’s performance \cite{bridle1992unsupervised,berthelot2020remixmatch}.

\begin{theorem}
\label{the}
For pseudo-label $p$ and the reversed pseudo-label $\overline{p}$ obtained by \textbf{Reverse Operation}, we show that the entropy of $\overline{p}$ is larger than that of $p$:
\begin{equation}
    \mathcal{H}(\overline{p})\geq \mathcal{H}(p),
    \label{eq:19}
\end{equation}
where $\mathcal{H}(\cdot)$ refers to the entropy.
\end{theorem}
\begin{proof}
We sort the sequence $p_{1},\dots,p_{n}$ in descending order and denote the sorted sequence as $p_{1}\geq\dots\geq p_{n}$ for simplicity. Considering the case where $p_{1}<\frac{1}{2}$ firstly, we prove a equivalent form of Theorem \ref{the}:
\begin{equation}
    \sum^{n}_{i=1}[p_{i}\log p_{i}-\frac{(1-p_{i})}{n-1}\log \frac{(1-p_{i})}{n-1}]\geq 0.
    \label{eq:16}
\end{equation}
We define the function as
\begin{equation}
    f(x)=x\log x-\frac{1-x}{n-1}\log\frac{1-x}{n-1},
\end{equation}
where $x\in [0,\frac{1}{2})$ by $\frac{1}{2}\geq p_{1}\geq\dots\geq p_{n}$.
The second derivative of this function is
\begin{equation}
    f''(x)=\frac{1}{x}-\frac{1}{(n-1)(1-x)}=\frac{(n-1)-nx}{x(n-1)(1-x)}
\end{equation}
Let $f''(x)\geq0$, we obtain $x\leq \frac{n-1}{n}$. Considering $n\geq 2$, the minimum of the term $\frac{n-1}{n}$ is $\frac{1}{2}$. By $x\leq \frac{1}{2}$, $f''(x)\geq0$ holds, which means the $f(x)$ is a convex function. Thus, by Jensen's Inequality, we have
\begin{equation}
    \frac{1}{n}\sum^{n}_{i=1}f(x_{i})\geq f(\frac{1}{n}\sum^{n}_{i=1}x_{i})
    \label{eq:leq12}
\end{equation}
Substituting in $x_{i}=p_{i}$, by Eq.  \eqref{eq:leq12}, we obtain
\begin{equation}
     \frac{1}{n}\sum^{n}_{i=1}(p_{i}\log p_{i}-\frac{1-p_{i}}{n-1}\log\frac{1-p_{i}}{n-1})\geq \frac{1}{n}\log\frac{1}{n}-\frac{1-\frac{1}{n}}{n-1}\log\frac{1-\frac{1}{n}}{n-1}=0
\end{equation}

Thus, Eq. \eqref{eq:16} holds when $p_{1}<\frac{1}{2}$. Next, we consider the case where $p_{1}\geq\frac{1}{2}$. Rewriting Eq.  \eqref{eq:19}, we obtain
\begin{equation}
    \sum^{n}_{i=1}p_{i}\log p_{i}\geq \sum^{n}_{i=1}\overline{p}_{i}\log \overline{p}_{i}.
    \label{eq:25}
\end{equation}
Denoting $\overline{p}_{1}=\frac{1-p_{n}}{n-1},\dots,\overline{p}_{n}=\frac{1-p_{1}}{n-1}$, we have 
\begin{equation}
    \frac{1}{n-1}\geq\overline{p}_{1}\geq\dots\geq \overline{p}_{n}.
    \label{eq:23}
\end{equation}
Let $\boldsymbol {a}=(\overline{p}_{1},\dots,\overline{p}_{n-1},\overline{p}_{n})$ and $\boldsymbol {b}= (\frac{1}{n-1},\dots,\frac{1}{n-1},0)$, by Eq. \eqref{eq:23} and $\sum^{n}_{i=1}\overline{p}_{i}=\sum^{n-1}_{i=1}\frac{1}{n-1}=1$,  we notice $\boldsymbol {a}$ is majorized by $\boldsymbol {b} $ 
($\boldsymbol {a}\prec \boldsymbol {b}$) 
\cite{marshall1979inequalities,arnold2012majorization}. Since the function $g(\boldsymbol{x})=\sum^{d}_{i=1} x_{i}\log(x_{i})$ is Schur-convex \cite{peajcariaac1992convex,roberts1993convex}, we have $g(\boldsymbol {a})\leq g(\boldsymbol {b})$ \cite{peajcariaac1992convex,roberts1993convex}, \ie, 
\begin{equation}
    \sum^{n}_{i=1}\overline{p}_{i}\log \overline{p}_{i}
    \leq 
    (n-1)\frac{1}{n-1}\log\frac{1}{n-1}=-\log(n-1).
    \label{eq:27}
\end{equation}
Next, rewriting the left term in Eq.  \eqref{eq:25}, we have
\begin{equation}
    \sum^{n}_{i=1}p_{i}\log p_{i}=p_{1}\log p_{1}+\sum^{n}_{i=2}p_{i}\log p_{i}.
    \label{eq:28}
\end{equation}
 Since $p_{2}+\dots+p_{n}=1-p_{1}$ and $g(x)=x\log x$ is a convex function, by Jensen’s Inequality, we obtain the minimum of $ \sum^{n}_{i=2}p_{i}\log p_{i} $ when $p_{2}=\dots=p_{n}=\frac{1-p_{1}}{n-1}$. Then, by Eq. \eqref{eq:28}, we have
 \begin{align}
     \sum^{n}_{i=1}p_{i}\log p_{i}  &\geq p_{1}\log  p_{1}+(\frac{1-p_{1}}{n-1}\log\frac{1-p_{1}}{n-1})(n-1) \nonumber \\
     &=p_{1}\log p_{1}+(1-p_{1})\log (1-p_{1})-(1-p_{1})\log (n-1) \nonumber \\
     &\geq-1-\frac{1}{2}\log (n-1) \label{eq:last}. \tag{using $p_{1}\log p_{1}+(1-p_{1})\log (1-p_{1})\geq -\log2$ and $1-p_{1}\leq \frac{1}{2}$ }
 \end{align}
Notice that by Eq. \eqref{eq:27} we have $\sum^{n}_{i=1}\overline{p}_{i}\log \overline{p}_{i}\leq-\log(n-1)$. Solving inequality
\begin{equation}
    -1-\frac{1}{2}\log (n-1)\geq -\log(n-1),
\end{equation}
we obtain that Eq. \eqref{eq:25} holds when $n\geq5$. Theorem \ref{the} now follows by combining the proofs for the cases where $p_{1}<\frac{1}{2}$ and $p_{1}\geq\frac{1}{2}$. In sum, for multi-classification tasks, we prove that when $n\geq5$, $\mathcal{H}(\overline{p})\geq \mathcal{H}(p)$ holds, \ie, \textit{Reverse Operation} could maximize the entropy of $p$. 
The proof for complementary label can be obtained by replacing $p$ and $\overline{p}$ in the above formulas with $q$ and $\overline{q}$, respectively. 
\end{proof}

Given the above proof, $\mathcal{D}$ and $\mathcal{A}$ are optimized to output predictions $\overline{p}$ and $\overline{q}$ with larger entropy, \ie,
\begin{equation}
\mathcal{H}[\mathbbm{E}_{u}(p)]+\mathcal{H}[\mathbbm{E}_{u}(q)]\leq\mathcal{H}[\mathbbm{E}_{u}(\overline{p})]+\mathcal{H}[\mathbbm{E}_{u}(\overline{q})].
\end{equation}
Thus it can be seen that RDA maximizes the objective Eq. \eqref{eq:obj} by aligning $\mathbbm{E}_{u}(p)$ to $\mathbbm{E}_{u}(\overline{q})$ and aligning $\mathbbm{E}_{u}(q)$ to $\mathbbm{E}_{u}(\overline{p})$ reciprocally, so as the input-output mutual information objective Eq. \eqref{eq:remix} could be maximized.

With \textit{Reverse Operation}, we can apply distribution alignment while ensuring that the relative relationship between classes in the class distribution can be utilized, so as RDA could achieve a more reasonable form of ``high entropy'' for the objective of distribution alignment without using prior about $C_{x}$.
So far, we construct hyperparameter-free Reciprocal Distribution Alignment (\textbf{RDA}), which is robust to SSL under both mismatched distributions and the conventional setting.
The whole algorithm is presented in Sec. \ref{sec:alg} of Supplementary Material.

\section{Experimental Setup}
\label{sec:exp}
We evaluate RDA on various standard benchmarks of SSL image classification task under diverse settings, including mismatched distributions (\ie, $C_{x}\not\approx C_{u}$) and the conventional SSL setting (\ie, $C_{x}\approx C_{u}$ and they are all balanced). 
Experiments show that RDA outperforms significantly over current state-of-the-art (SOTA) SSL methods under most settings. 
We also conduct further ablation studies on the effectiveness of each components in our method.

\subsection{Datasets}
RDA is evaluated on four datasets used in SSL widely: CIFAR-10/100~\cite{krizhevsky2009learning},
 STL-10 \cite{coates2011an} and mini-ImageNet~\cite{vinyals2016matching}.
 CIFAR-10/100, are composed of 60,000 images from 10/100 classes.
 Both of them are divided into training set with 50,000 images and test set with 10,000 images. 
 STL-10 is composed of 5,000 labeled images and 100,000 unlabeled images which extracted from a broader distribution.
 mini-ImageNet is a subset of ImageNet~\cite{deng2009imagenet} consisting of 100 classes, and each class has 600 images.
 \subsection{Settings of $C_{x}$ and $C_{u}$}
 \label{sec:cxcu}
 In addition to the conventional matched setting (\ie, both $C_{x}$ and $C_{u}$ are balanced), we verify the efficacy of our method in more realistic mismatched scenarios, as discussed in Sec. \ref{sec:md}. In view of the complexity of this problem, we mainly use the following three scenarios to summary our experimental protocol:
 \begin{itemize}
  \item[$\bullet$]
  Training with imbalanced $C_{x}$ and balanced $C_{u}$. We are interested in the impact of mismatched distributions resulting from this simple setting. A graphical explanation of this setting is shown in Fig. \ref{fig:mis1}.
  \item[$\bullet$]
 Training with mismatched and imbalanced $C_{x}$, $C_{u}$, 
which is shown in  Fig. \ref{fig:mis2}. This  challenging setting can fully test the robustness of RDA.
     \item[$\bullet$]
 Training with balanced $C_{x}$ and imbalanced $C_{u}$.
\end{itemize}

For experiments in above scenarios, we randomly select samples from dataset to construct imbalanced $C_{x}$ and $C_{u}$.
For $C_{x}$, the number of labeled data $N_{i}$ in each class is fixed by $N_{0}$. $N_{i}$ is calculated as $N_{i}=N_{0}\times \gamma_x^{-\frac{i-1}{n-1}}$, where $n$ is the number of classes and $i\in (1,\dots,n)$. For fairness, we hold $N_{0}$ and search a proper $\gamma_x$ for each $N_{i}$ to keep the total number of labeled data consistent with we set. Details on searching for $\gamma_x$ are shown in  Sec. \ref{app:mcd} of Supplementary Material.

Specially, $C_{u}$ is constructed in a form similar to reversely ordered $C_{x}$ 
for more challenging setting. After a random selection of unlabeled data from dataset, the remaining data is seen as unlabeled data. The number of unlabeled data $M_{i}$ of each class is fixed by: $M_{i}=M_{0}\times \gamma^{-\frac{n-i}{n-1}}$, where $M_{0}=5000$ in CIFAR-10, $M_{0}=500$ in mini-ImageNet. In this way, we construct $C_{u}$ as a \textit{``reversed''} version of $C_{x}$ as shown in Fig. \ref{fig:mis2}. Likewise, DARP's protocol  \cite{kim2020distribution} also produces datasets with mismatched distributions from CIFAR-10 and STL-10. So we also make a fair comparison with DARP  under this protocol. More details about DARP's protocol can be found in  Sec. \ref{app:darp} of Supplementary Material.

 \subsection{Baselines}
We compare RDA mainly with three recent state-of-the-art SSL methods: (1) FixMatch~\cite{sohn2020fixmatch}, combining consistency regularization and entropy minimization; (2) FixMatch with distribution
alignment~\cite{berthelot2020remixmatch}; (3) CoMatch~\cite{li2021comatch}, combining graph-based contrastive learning and consistency regularization. 
Moreover,  we provide more comparisons with MixMatch \cite{berthelot2019mixmatch}, AlphaMatch \cite{gong2021alphamatch}, and DARP \cite{kim2020distribution}.
\subsection{Implementation Details}
\label{sec:id}

Unless noted otherwise,
we adopt Wide ResNet~\cite{zagoruyko2016wide} 
and Resnet-18 \cite{he2016deep}
as the backbone for experiments. 
For specific, WRN-28-2 is used for CIFAR-10, WRN-28-8 is used for CIFAR-100 and Resnet-18 is used for STL-10/mini-ImageNet.
Following~\cite{sohn2020fixmatch}, RandAugment~\cite{cubuk2020randaugment} is used for strong augmentation. For simplicity, we train models using SGD with a momentum of 0.9 and a weight decay of 0.0005 in all experiments. In addition, we use a learning rate of 0.03 with cosine decay schedule to train the models for 1024 epochs.  
For hyper-parameters, we set $\mu =7, B=64, \lambda_{a}=\lambda_{cd}=\lambda_{ca}=1$ for all experiments.
Particularly, we report the results averaged on five folds and the standard deviation is calculated.

\section{Results and Analysis}
\subsection{Conventional Setting (Matched Distributions)}
For a fair comparison with baseline SSL methods, we conduct experiments in the conventional setting, \ie, both $C_{x}$ and $C_{w}$ are balanced. We test the accuracy of RDA on CIFAR-10, mini-ImageNet, and  STL-10 by varying the number of labeled data.
Tab. \ref{table:coven} shows that the performance of RDA is compatible to (if not better than) that of the conventional SSL methods under matched class distributions.
This results also confirm our view that with our design, the distribution alignment alone is enough for pseudo-labeling.
RDA outperforms CoMatch by 3.60\%  when labels are scarce (with 20 labels). Moreover, on datasets with more classes, our method consistently achieves improvement on accuracy than the best baseline, \eg, 46.91\% (ours) vs 43.72\% (CoMatch) on mini-ImageNet with 1000 labels. The superior performance benefits from RDA, which improves pseudo-labels with the co-regularization of complementary class distribution and utilizes the entire unlabeled data, whereas low-confidence samples are filtered out in \cite{sohn2020fixmatch,li2021comatch}. 

\subsection{Mismatched Distributions}
\label{sec:res}
\textbf{Imbalanced $C_{x}$ and Balanced $C_{u}$}.
We keep balanced distribution in the unlabeled data and vary $N_{0}$ to change the imbalance degree of $C_{x}$ while the total number of labeled data remains unchanged in the way described in Sec. \ref{sec:cxcu}. Tab. \ref{table:mismatch1} shows the results on CIFAR-10, CIFAR-100, and mini-ImageNet.  
RDA outperforms all baseline methods by a large margin. \eg, on CIFAR-10, with 100 labels and $N_{0}=80$, RDA outperforms  FixMatch by 7.43\% and CoMatch  by 52.09\%. 
We witness that mismatched $C_{x}$ and $C_{u}$ significantly decrease the models' performance. Notably, the traditional distribution alignment, assuming the labeled and unlabeled data share the same distribution, significantly degrades the performance of model when the distributions mismatch, whereas our method improves this situation by utilizing guidance of distribution information without any prior.  
As shown in  Figs. \ref{fig:resa} and \ref{fig:resc},  RDA resists the impact of imbalanced $C_{x}$ and computes a more balanced pseudo-label distribution than FixMatch, demonstrating the effectiveness of RDA in this mismatched distributions scenario. Additionally, Figs. \ref{fig:resb} and \ref{fig:resd} show that the predictions of RDA are not necessarily more confident than that of FixMatch, but RDA reduces the overfitting on false pseudo-labels, \ie, RDA  is not as overconfident as FixMatch on pseudo-labels that may be wrong.
Thanks to no requirement of prior  about the labeled data distribution, RDA can be safely applied to this scenario without being overwhelmingly affected by distribution gap, thus exhibiting robust performance.

\begin{table}[t]
    \centering
  \caption{Results of accuracy (\%) in the conventional matched SSL setting.
  Results with $\ast$ are copied from CoMatch~\cite{li2021comatch} and with $\dagger$ are copied from AlphaMatch~\cite{gong2021alphamatch}. Results of other baselines are based on our reimplementation. }
  \label{table:coven}
    \setlength{\tabcolsep}{1.58mm}{
        \fontsize{5.7}{5}\selectfont
      \begin{tabular}{@{}l|cccc|c|c@{}}     
      \toprule
      \multirow{2}{*}{Method} & \multicolumn{4}{c|}{CIFAR-10}  & \multicolumn{1}{c|}{mini-ImageNet} & \multicolumn{1}{c}{STL-10} \\  \cmidrule(lr){2-5} \cmidrule(lr){6-6}   \cmidrule(lr){7-7}  
      & 20 labels            & 40 labels       & 80 labels  & 100 labels             & 1000 labels & 1000 labels\\ \cmidrule(r){1-1} \cmidrule(lr){2-5}  \cmidrule(lr){6-6}   \cmidrule(lr){7-7}   
      
      $\mbox{MixMatch}^{\ast}$              & 27.84$\pm$10.63                               & 51.90$\pm$11.76                               & 80.79$\pm$1.28                    & -             & - & 38.02$\pm$8.29  \\ 
        $\mbox{AlphaMatch}^{\dagger}$              & -                               & 91.35$\pm$3.38                               &-                    & -             & - & -  \\ \cmidrule(r){1-1} \cmidrule(lr){2-5}  \cmidrule(lr){6-6}   \cmidrule(lr){7-7} 
      $\mbox{FixMatch}$                    & 84.97$\pm$10.37       & 89.18$\pm$1.54            & 91.99$\pm$0.71      & 93.14$\pm$0.76  & 39.03$\pm$0.66    & $\mbox{65.38$\pm$0.42}^\ast$  \\
      $\mbox{CoMatch}$              & 88.43$\pm$7.22                               & 93.21$\pm$1.55                               & 94.08$\pm$0.31                    & \textbf{94.55$\pm$0.27}               & 43.72$\pm$0.58  & $\mbox{79.80$\pm$0.38}^\ast$   \\\cmidrule(r){1-1} \cmidrule(lr){2-5}  \cmidrule(lr){6-6}   \cmidrule(lr){7-7}  

      RDA              & \textbf{92.03$\pm$2.01}            & \textbf{94.13$\pm$1.22}             & \textbf{94.24$\pm$0.42}           & 94.35$\pm$0.25       & \textbf{46.91$\pm$1.16}   & \textbf{82.63$\pm$0.54}                    \\ 
      \bottomrule
      \end{tabular}
      }
\end{table}
\begin{table}[t]
  \caption{Results of accuracy (\%) in the mismatched scenario with imbalanced $C_{x}$ (\ie, alter $N_{0}$) and balanced $C_{u}$. Experiments are conducted
  on CIFAR-10, CIFAR-100 and mini-ImageNet varying the number of labels and $N_{0}$. Baseline methods are using our reimplementation.  Results with \textbf{DA} are achieved by combining the original \textit{distribution alignment} in~\cite{berthelot2020remixmatch}. 
  \textbf{Note that CoMatch \cite{li2021comatch} also integrates DA technique.}
  }
  \vskip 0in
    \centering
  \resizebox{\textwidth}{!}{  
      \label{table:mismatch1}
     \!\!\! \begin{tabular}{@{}l|cccc|cc|cc@{}}     
      \toprule[1.2pt]
      \multirow{3}{*}{Method} & \multicolumn{4}{c|}{CIFAR-10} & \multicolumn{2}{c|}{CIFAR-100} & \multicolumn{2}{c}{mini-ImageNet}  \\  \cmidrule(lr){2-3}  \cmidrule(lr){4-5} \cmidrule(lr){6-6} \cmidrule(lr){7-7} \cmidrule(lr){8-9}  
      & \multicolumn{2}{c}{40 labels}   & \multicolumn{2}{c|}{100 labels} & \multicolumn{1}{c}{400 labels} & \multicolumn{1}{c|}{1000 labels} & \multicolumn{2}{c}{1000 labels} \\  \cmidrule(lr){2-3}  \cmidrule(lr){4-5}  \cmidrule(lr){6-6} \cmidrule(lr){7-7}   \cmidrule(lr){8-9}  
      & $N_{0}=10$            & 20                       & 40 & 80 &  40 & 80      & 40 &80   \\ \cmidrule(r){1-1} \cmidrule(lr){2-3}  \cmidrule(lr){4-5}\cmidrule(lr){6-6} \cmidrule(lr){7-7}   \cmidrule(lr){8-9}   
      FixMatch    & 85.72$\pm$0.93        & 76.53$\pm$3.03     & 93.01$\pm$0.72  & 71.57$\pm$1.88    & 25.66$\pm$0.46    & 40.22$\pm$1.00    & 36.20$\pm$0.36  & 28.33$\pm$0.41 \\
      FixMatch w. DA  & 71.23$\pm$1.25  & 47.85$\pm$1.99  & 56.78$\pm$1.28    & 34.18$\pm$0.86  & 22.66$\pm$1.53    & 31.06$\pm$0.51   &  33.87$\pm$0.40 & 23.53$\pm$0.72\\
      CoMatch    & 60.27$\pm$3.22        & 39.48$\pm$2.20     & 52.82$\pm$2.03  & 26.91$\pm$0.75   & 23.97$\pm$0.62    &  28.35$\pm$1.20   & 30.24$\pm$1.37  & 21.47$\pm$0.86 \\ \cmidrule(r){1-1} \cmidrule(lr){2-3}  \cmidrule(lr){4-5}  \cmidrule(lr){6-6} \cmidrule(lr){7-7}   \cmidrule(lr){8-9}   

      RDA    & \textbf{92.57$\pm$0.53}     & \textbf{81.78$\pm$6.44}     & \textbf{94.23$\pm$0.36}      &  \textbf{79.00$\pm$2.67}      & \textbf{30.86$\pm$0.78}            & \textbf{41.29$\pm$0.43}       & \textbf{42.73$\pm$0.84}    & \textbf{36.73$\pm$1.01}     \\ 
      \bottomrule[1.2pt]
      \end{tabular}
  }
\end{table}
\newcommand{\mz}{2.9cm}
\newcommand{\wid}{0.24}
\begin{figure*}[t]
  \centering
  \subfigure[(40, 20, 1)]{
  \includegraphics[width=\mz]{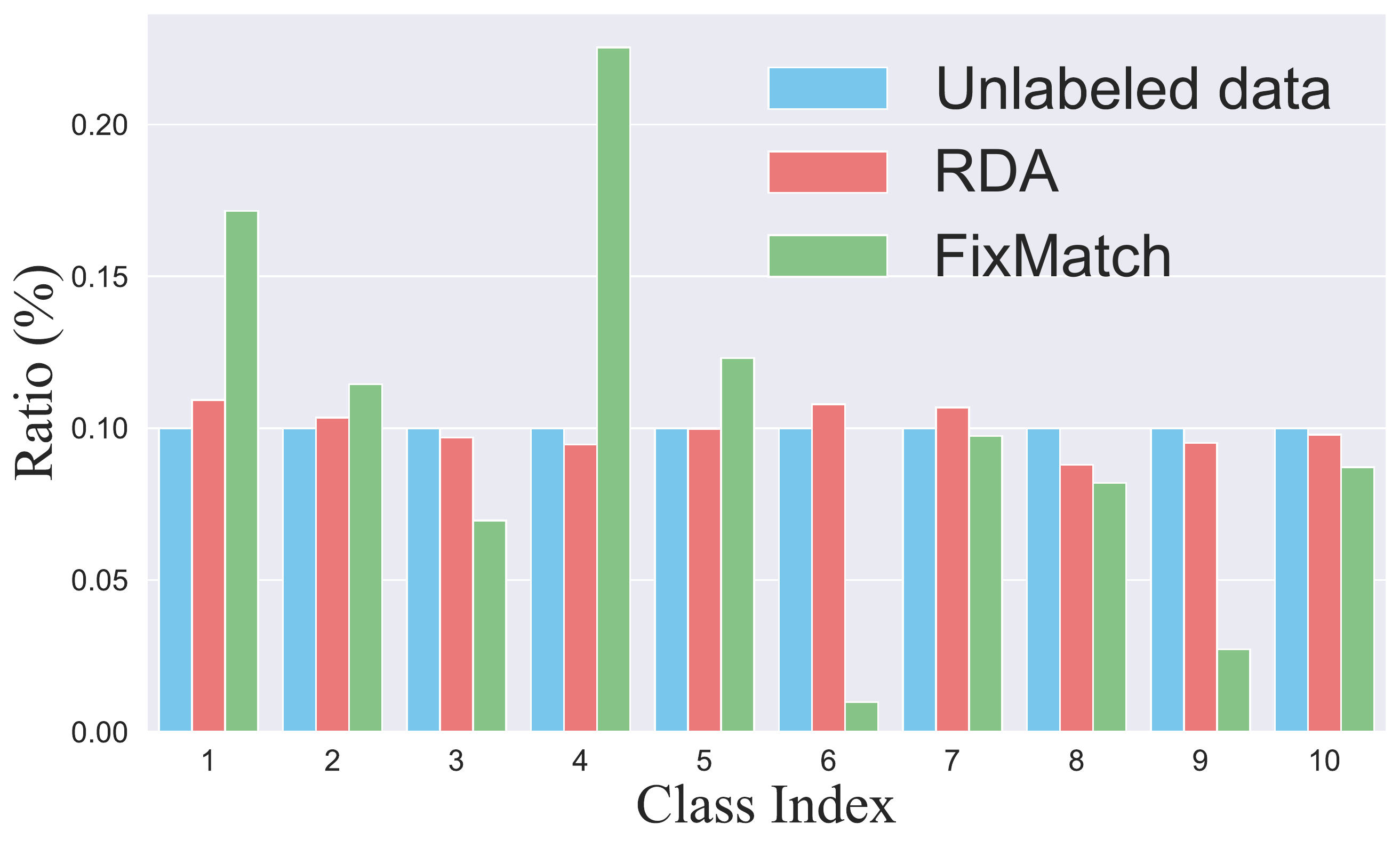}
    \label{fig:resa}
  }
  \hspace{-3mm}
  \subfigure[(40, 20, 1)]{
  \includegraphics[width=\mz]{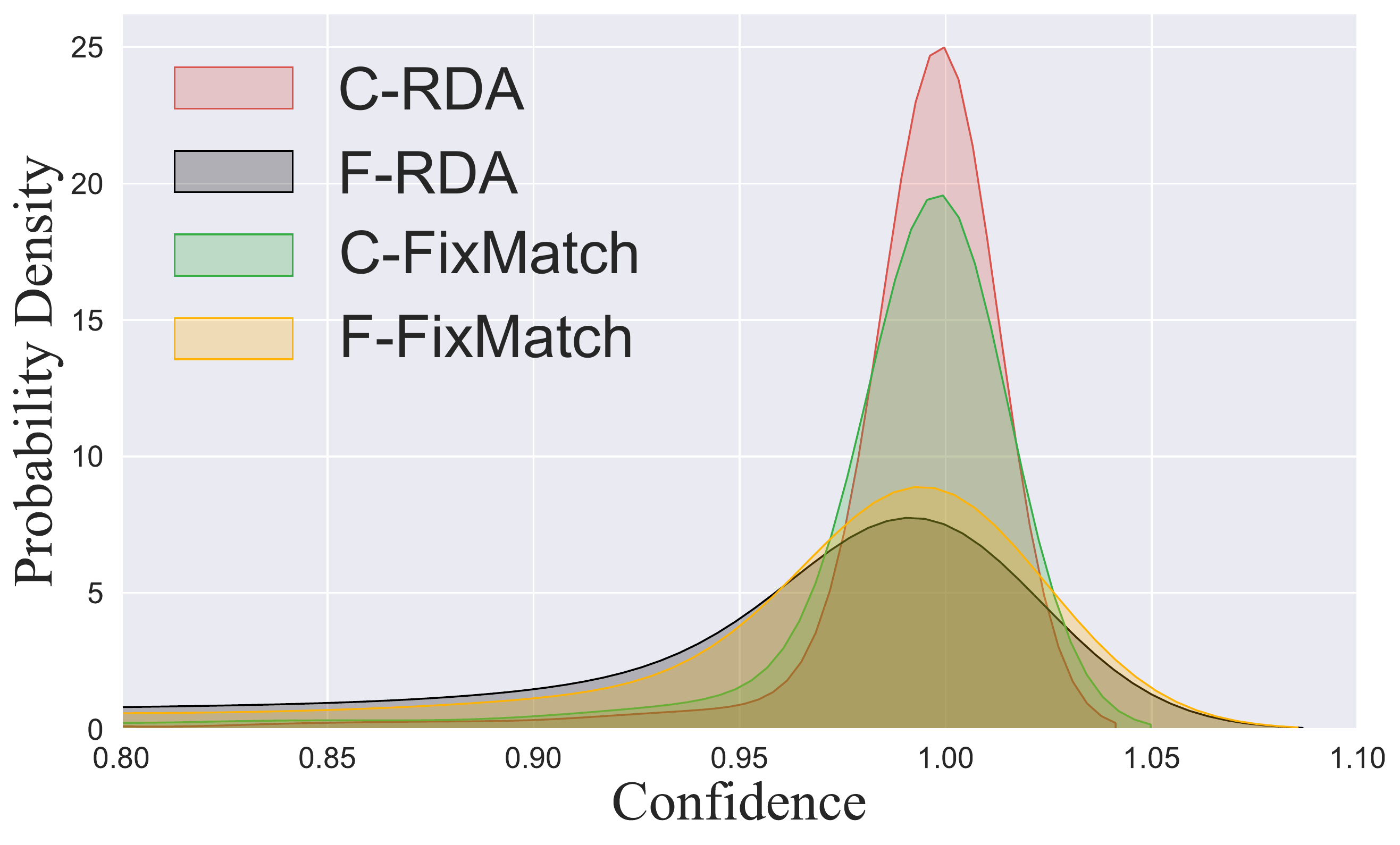}
  \label{fig:resb}
  }
  \hspace{-3mm}
  \subfigure[(100, 80, 1)]{
  \includegraphics[width=\mz]{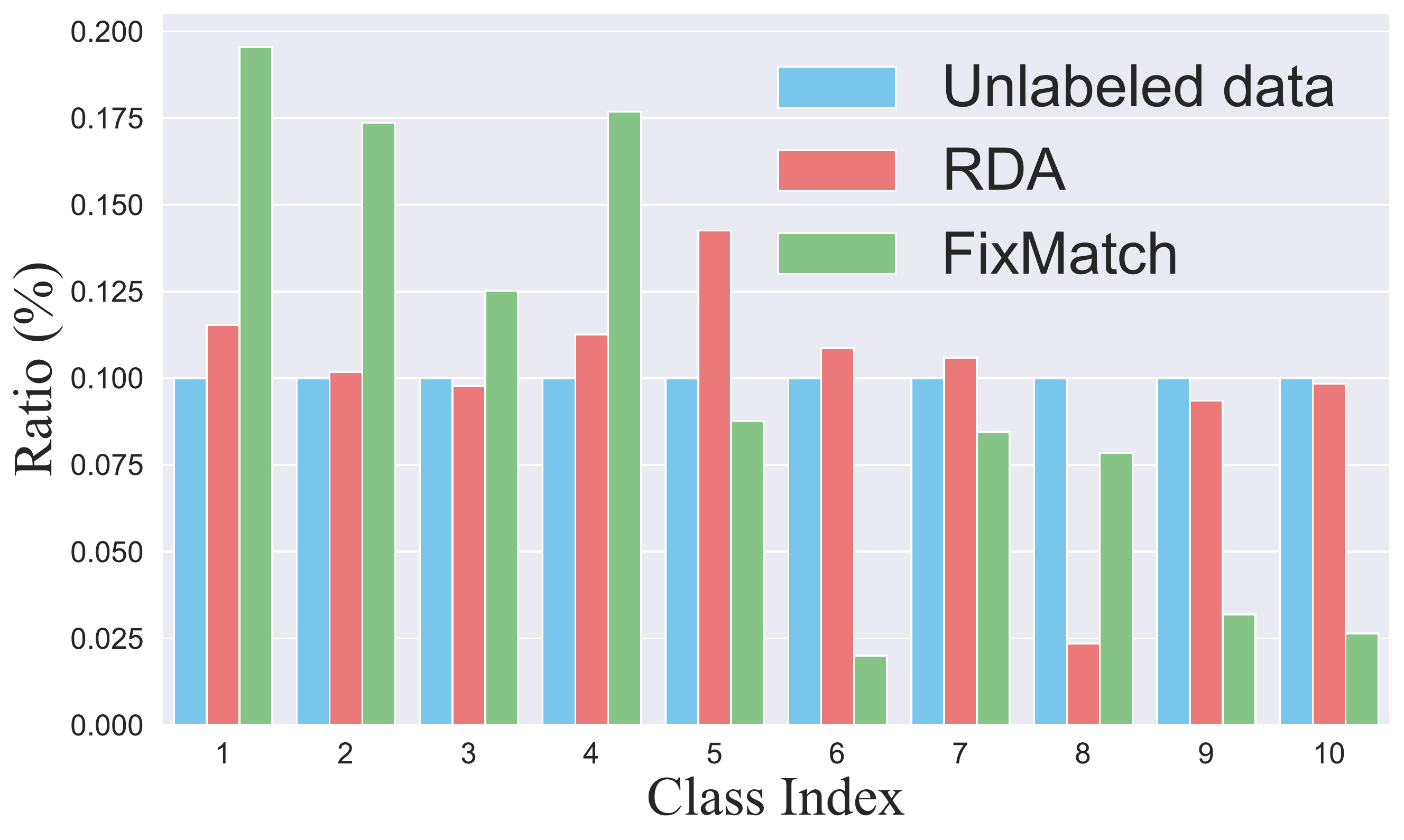}
  \label{fig:resc}
  }
  \hspace{-3mm}
  \subfigure[(100, 80, 1)]{
  \includegraphics[width=\mz]{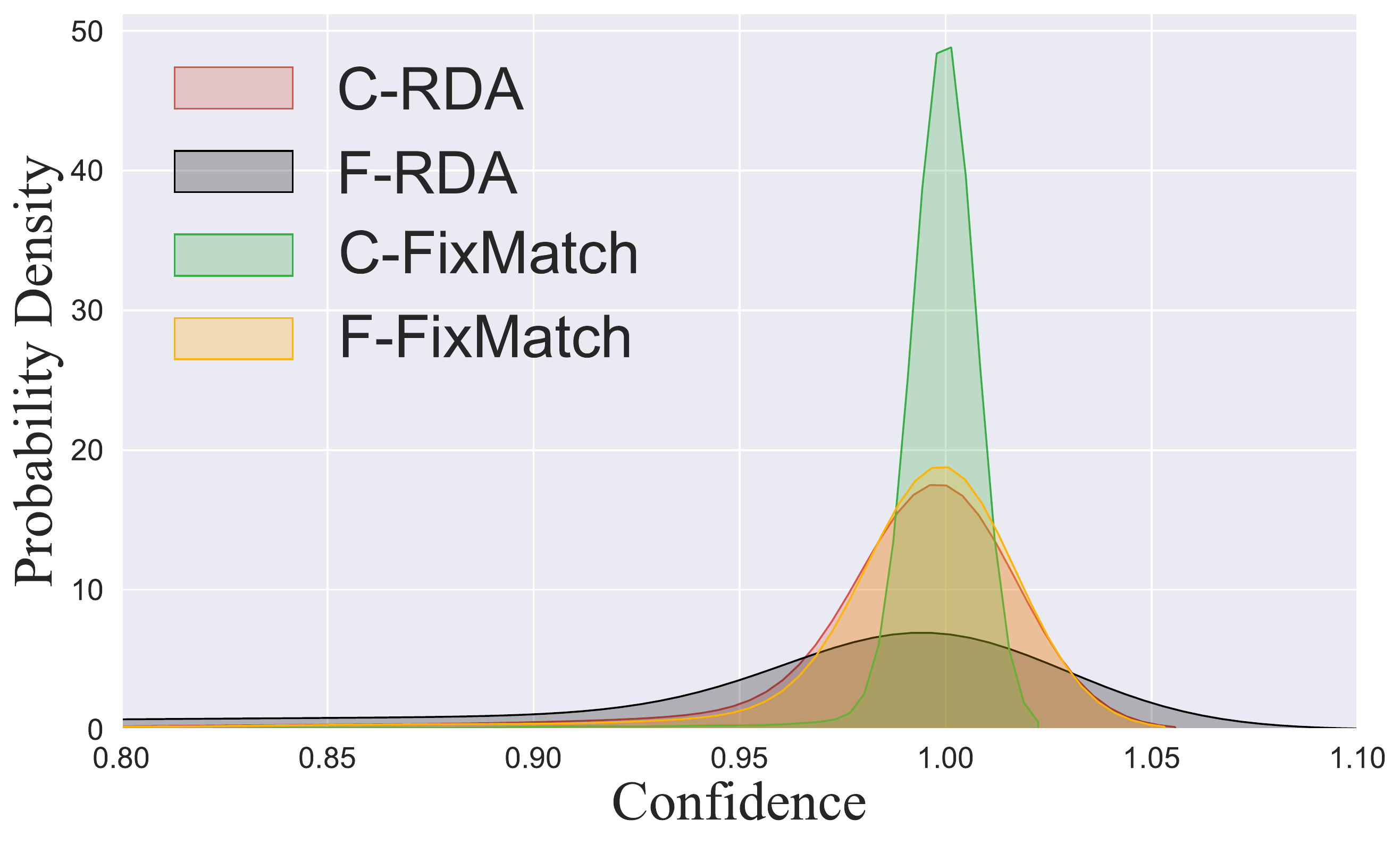}
  \label{fig:resd}
  }
  \subfigure[(40, 10, 5)]{
  \includegraphics[width=\mz]{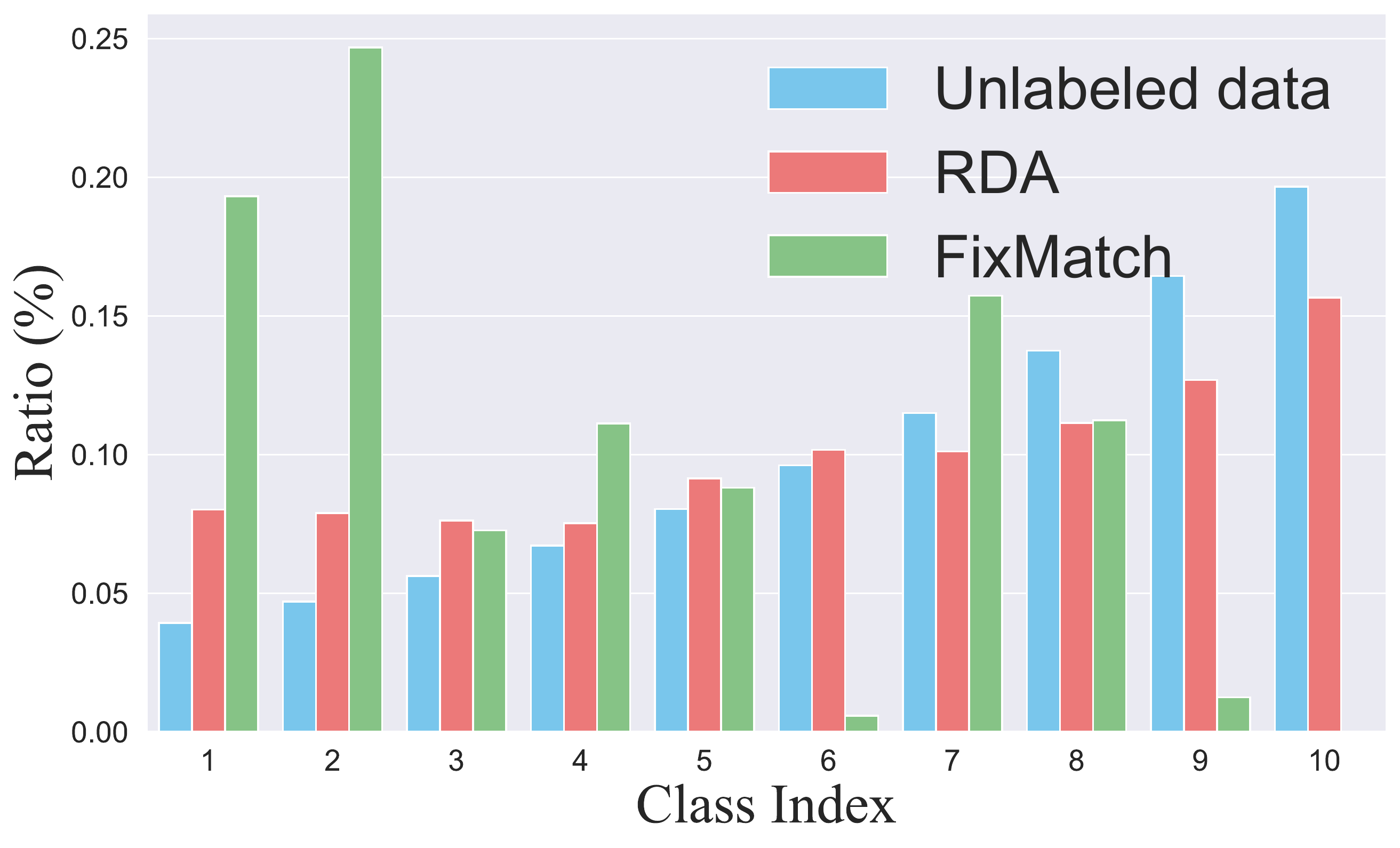}
  \label{fig:rese}
  }
  \hspace{-3mm}
  \subfigure[(40, 10, 5)]{
  \includegraphics[width=\mz]{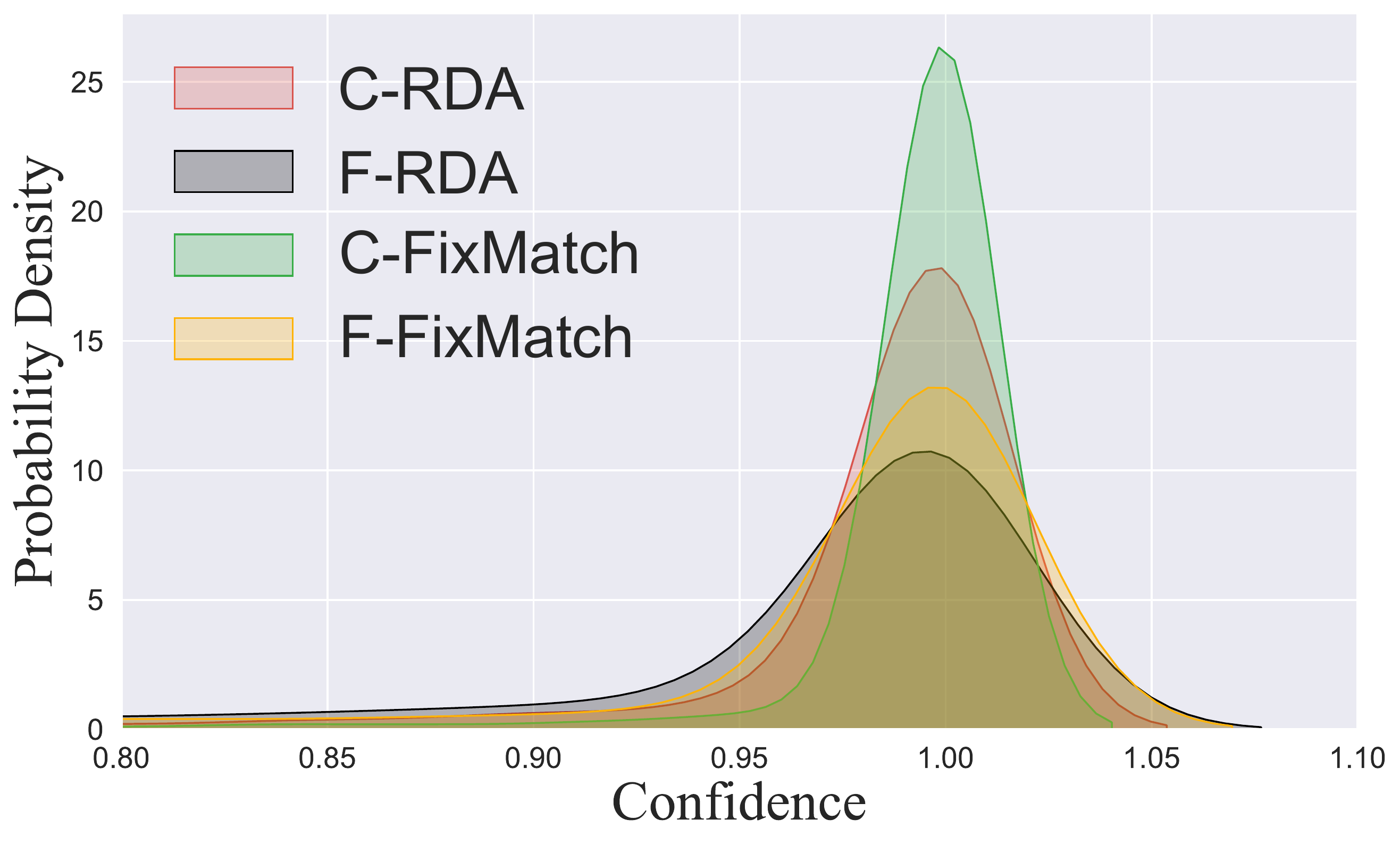}
  \label{fig:resf}
  }
  \hspace{-3mm}
  \subfigure[(100, 40, 10)]{
  \includegraphics[width=\mz]{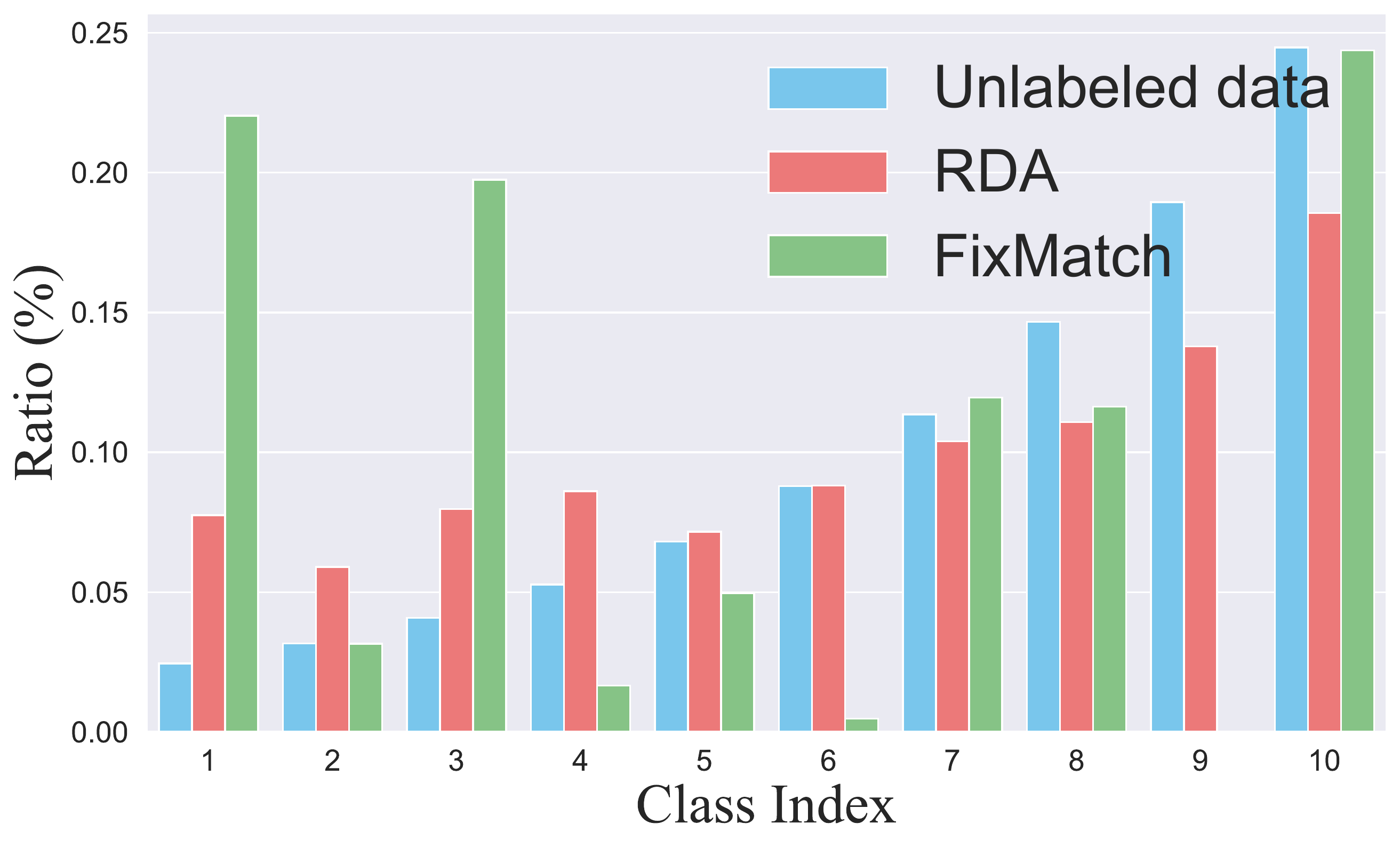}
  \label{fig:resg}
  }
  \hspace{-3mm}
  \subfigure[(100, 40, 10)]{
  \includegraphics[width=\mz]{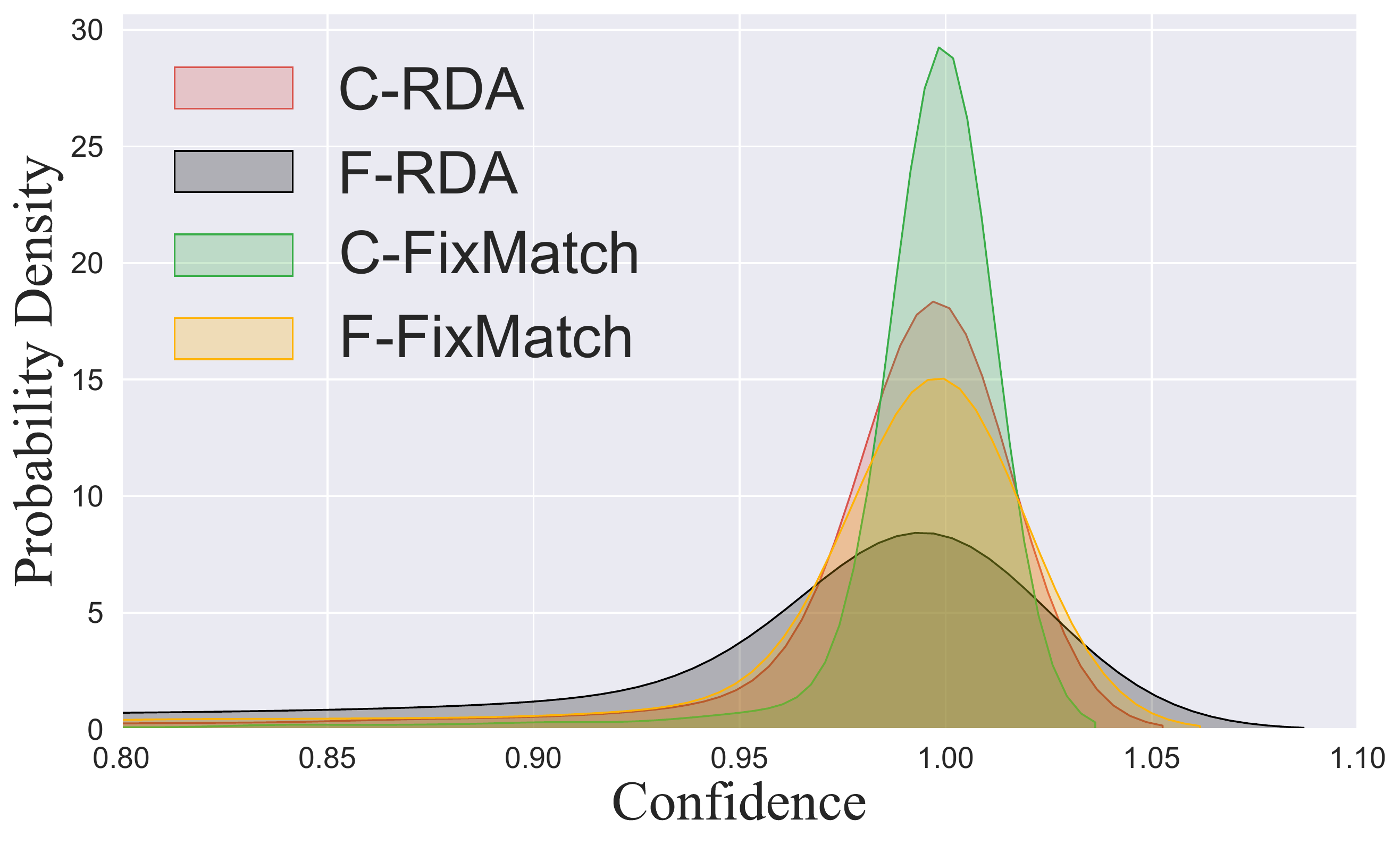}
  \label{fig:resh}
  }
  \caption{In the caption, ($x$,$y$,$z$) denotes (labels, $N_{0}$, $\gamma$). In (a), (b), (c) and (d), $C_{x}$ is imbalanced and $C_{u}$ is balanced. In (e), (f), (g) and (h), $C_{x}$ and $C_{u}$ are imbalanced and they mismatch. In (a), (c), (e) and (g),  the x-axis represents the index of classes in CIFAR-10 and the y-axis represents the ratio of label to the total. \textit{RDA}/\textit{FixMatch} in figures indicates the class predictions from RDA/FixMatch and \textit{Unlabeled data} indicates the ground-truth label of unlabeled samples. In (b), (d), (f) and (h), the x-axis represents the confidence of prediction from RDA/FixMatch and the y-axis represents the probability density of confidence estimated by kernel density estimation (KDE). \textit{C-\texttt{X}} and \textit{F-\texttt{X}} indicate the correct and false class predictions of \textit{\texttt{X}}, respectively.}
  \label{fig:res}
  \vskip 0in
\end{figure*}
\begin{table}[t]
\centering

  \caption{Results of accuracy (\%) with mismatched and imbalanced $C_{x}$, $C_{u}$ (\ie, alter both $N_{0}$ and $\gamma$ at the same time).  Baseline methods are based on our reimplementation. We omit the results of baselines that combine DA considering their poor performance.}

  \vskip 0in
      \label{table:mismatch2}
        \scriptsize
    \setlength{\tabcolsep}{2.35mm}{
      \begin{tabular}{@{}l|cccc|c@{}}     
      \toprule
      \multirow{3}{*}{Method} & \multicolumn{4}{c|}{CIFAR-10}  & \multicolumn{1}{c}{mini-ImageNet} \\  \cmidrule(lr){2-3}  \cmidrule(lr){4-5}   \cmidrule(lr){6-6} 
      & \multicolumn{2}{c}{40 labels, $N_{0}=$10}     & \multicolumn{2}{c|}{100 labels, $N_{0}=$40}   &  1000 labels, $N_{0}=$40  \\  \cmidrule(lr){2-3}  \cmidrule(lr){4-5}    \cmidrule(lr){6-6} 
      & $\gamma=2$    & 5 & 5   & 10   & 10 \\ \cmidrule(r){1-1} \cmidrule(lr){2-3}  \cmidrule(lr){4-5}    \cmidrule(lr){6-6} 
      FixMatch                     & 74.97$\pm$5.80              & 64.62$\pm$6.13                               & 58.72$\pm$3.61                    & 57.49$\pm$4.56                    & 21.40$\pm$0.53 \\

      RDA              & \textbf{88.58$\pm$4.05}            & \textbf{79.90$\pm$2.80}             & \textbf{79.33$\pm$1.37}           & \textbf{70.93$\pm$2.91}          & \textbf{25.99$\pm$0.19}        \\ 
      \bottomrule
      \end{tabular}
      }
\end{table}
\noindent\textbf{Mismatched and Imbalanced $C_{x}$, $C_{u}$}.
Results of the more challenging setting are summarized in Tab.  \ref{table:mismatch2}. While FixMatch and CoMatch fail to correct the
severely biased prediction on unlabeled data caused by reversely ordered labeled data, RDA shows its superior performance in this setting and outperforms baseline methods significantly once again.
As shown in Figs. \ref{fig:rese} and \ref{fig:resg}, while imbalanced and mismatched $C_{x}$, $C_{u}$ lead to strong bias on FixMatch's predictions, RDA shows extraordinary robustness to this scenario. 
In contrast to FixMatch, RDA prevents overfitting of false pseudo-labels, as shown in Figs. \ref{fig:resf} and \ref{fig:resh}.

\begin{wraptable}{r}{0.4\textwidth}
    \scriptsize
    \centering
    \vskip -0.15in
            \caption{Accuarcy (\%) on CIFAR-10 with balanced $C_{x}$ and imbalanced $C_{u}$ (\ie, alter $\gamma$).}
            \label{table:bi}
            \vskip 0.1in
            \setlength{\tabcolsep}{2.65mm}{
            \begin{tabular}{@{}l|c@{}}      
            \toprule
            \multirow{1}{*}{Method} & {40 labels, $\gamma=200$}   \\ \cmidrule(r){1-1}   \cmidrule(lr){2-2} 

            \mbox{FixMatch w. DA}  & {41.37$\pm$1.22}  \\
            \mbox{CoMatch}   & 38.85$\pm$2.19  \\ \cmidrule(r){1-1} \cmidrule(lr){2-2}
            RDA              & \textbf{46.50$\pm$1.07}  \\ 
            \bottomrule
            \end{tabular}
            }
\vskip -1em            
\end{wraptable}

\noindent\textbf{Balanced $C_{x}$ and Imbalanced $C_{u}$}. 
As shown in Tab. \ref{table:bi}, 
RDA shows the compatibility to this scenario and also outperforms baselines combining distribution alignment. Mismatched distributions caused by balanced $C_{x}$ and imbalanced $C_{u}$ also lead to poor performance of methods with original distribution alignment. 

\noindent\textbf{Other Mismatched Settings.}
We also show results of RDA within the DARP's protocol averaged on all five runs. 
As shown in Tab. \ref{table:darp}, RDA consistently outperforms current class-imbalanced SSL method DARP \cite{kim2020distribution} and shows the largest gains in all settings with mismatched $C_{x}$ and $C_{u}$. 
More discussions on generalized settings of mismatched distributions can be found in Sec. \ref{app:add} of Supplementary Material.

\begin{table}[t]
\centering
  \scriptsize
  \caption{Accuracy (\%) under DARP’s protocol (see Sec. \ref{app:darp} of Supplementary Material for more details and baselines).  
 WRN-28-2 is adopted as the backbone for all datasets.
  }

  \vskip 0in
      \label{table:darp}
      \setlength{\tabcolsep}{1.3mm}{
      \begin{tabular}{@{}l|cccc|cc@{}}     
      \toprule
      \multirow{3}{*}{Method} & \multicolumn{4}{c|}{CIFAR-10 ($\gamma_{l}=$100)} &\multicolumn{2}{c}{STL-10 ($\gamma_{l}\neq \gamma_{u}$)} \\  \cmidrule(lr){2-5}  \cmidrule(lr){6-7}
      & $\gamma_{u}=1$    & 50 & 150  & 100 (reversed)  &  $\gamma_{l}=10$ & 20 \\ \cmidrule(r){1-1} \cmidrule(lr){2-5}     \cmidrule(lr){6-7}
      FixMatch                     & 68.90$\pm$1.95              & 73.90$\pm$0.25                               & 69.60$\pm$0.60                     & 65.50$\pm$0.05      & $\mbox{72.90$\pm$0.09}$ & $\mbox{63.40$\pm$0.21}$           \\
  
     DARP              & 85.40$\pm$0.55       & 77.30$\pm$0.17   &72.90$\pm$0.24  & 74.90$\pm$0.51   &77.80$\pm$0.33  & 69.90$\pm$0.40         \\ \cmidrule(r){1-1} \cmidrule(lr){2-5}      \cmidrule(lr){6-7}
      RDA              & \textbf{93.35$\pm$0.24}            & \textbf{79.77$\pm$0.06}             & \textbf{74.48$\pm$0.24}           & \textbf{79.25$\pm$0.52}     & \textbf{87.21$\pm$0.44}             &\textbf{83.21$\pm$0.52}       \\ 
      \bottomrule
      \end{tabular}
      }
  \vskip 0in
\end{table}

\subsection{Ablation Study}
\label{sec:ab}
To prove the effectiveness of each component in RDA, we conduct ablation studies on CIFAR-10 using consistent experimental setup with Sec. \ref{sec:id}. We mainly conduct experiments in three settings described in Sec. \ref{sec:cxcu}
and change the strategy performing distribution alignment from each direction as follows:

\noindent$\bm{\mathbbm{E}_{u}(p)\Rightarrow\mathbbm{E}_{u}(\overline{q})}$. 
We keep Eq. \eqref{eq:da1} and discard Eq. \eqref{eq:da2}. \ie, we align distribution of class predictions to ``reversed'' distribution of complementary predictions.

\noindent$\bm{\mathbbm{E}_{u}(q)\Rightarrow\mathbbm{E}_{u}(\overline{p})}$. 
We keep Eq. \eqref{eq:da2} and discard Eq. \eqref{eq:da1}. \ie, we align distribution of complementary predictions to ``reversed'' distribution of class predictions.

As shown in Tab. \ref{table:ab}, the performance of default RDA in mismatched distributions is dominant. RDA helps the model better maximize the objective Eq. \eqref{eq:obj} while obtaining helpful guidance information of class distribution without prior.

\begin{table}[t]
    \centering
  \caption{Accuracy (\%) of ablation studies on CIFAR-10 with two alternative alignment strategies. ``/'' represents the conventional setting and $\gamma=1$ represents balanced $C_{u}$. }
      \label{table:ab}
      \vskip 0in
      \scriptsize
      \setlength{\tabcolsep}{0.88mm}{
    \begin{tabular}{@{}l|cccccc@{}}     
      \toprule
      \multirow{3}{*}{Method} & \multicolumn{3}{c}{40 labels}   &\multicolumn{3}{c} {100 labels}   \\  \cmidrule(lr){2-4}  \cmidrule(lr){5-7}   
      & $N_{0},\gamma=\textrm{/}$ & 20,1 & 10, 5  &/\ &  80,1 & 40, 10    \\ \cmidrule(r){1-1} \cmidrule(lr){2-4}  \cmidrule(lr){5-7} 
      $\mathbbm{E}_{u}(p)\Rightarrow\mathbbm{E}_{u}(\overline{q})$&91.88$\pm$1.46       & 73.54$\pm$3.44   & 74.83$\pm$2.99 &94.14$\pm$0.52  & 54.88$\pm$11.79   & 62.96$\pm$3.43  \\
      $\mathbbm{E}_{u}(q)\Rightarrow\mathbbm{E}_{u}(\overline{p})$  &93.35$\pm$0.12  & 58.90$\pm$3.50 & 57.38$\pm$3.63 
      & \textbf{94.60$\pm$0.08} & 54.26$\pm$4.34   & 55.39$\pm$14.14   \\ \cmidrule(r){1-1} \cmidrule(lr){2-4}  \cmidrule(lr){5-7}   
      RDA  &\textbf{94.13$\pm$1.22} & \textbf{81.78$\pm$6.44}     & \textbf{79.90$\pm$2.88}  & {94.35$\pm$0.25}  & \textbf{79.00$\pm$2.67}  & \textbf{70.93$\pm$2.91} \\ 
      \bottomrule
      \end{tabular}
      }
    \vskip 0in
\end{table}
\section{Conclusion}
In this work, we propose a semi-supervised learning approach which is robust to both the conventional SSL and SSL in mismatched distributions. First, we describe a scenario that has not been discussed extensively by recently-proposed SSL work: mismatched distributions. Second, we improve distribution alignment by proposed RDA so that this technique could be applied into mismatched scenario safely. Then we show RDA results in a form of maximizing the input-out mutual information without any prior information. Finally, we demonstrate that our method outperforms existing baselines significantly under various scenarios.

\noindent \textbf{Acknowledgements.} This work is supported by projects from NSFC Major Program (62192783), CAAI-Huawei MindSpore (CAAIXSJLJJ-2021-042A), China Postdoctoral Science Foundation (2021M690609), Jiangsu NSF (BK20210224), and CCF-Lenovo Bule Ocean. Thanks to Prof. Penghui Yao's helpful discussions. 

\clearpage
%
%
\bibliographystyle{splncs04}
\bibliography{egbib}

\clearpage
\appendix


\section*{Supplementary Material}

\section{Algorithm}
\label{sec:alg}

Pseudo-code of RDA is shown in Algorithm \ref{a}.
\vspace{ -1em}
\begin{algorithm}[h]
      \caption{RDA: Reciprocal Distribution Alignment} 
      \LinesNumbered 
      \label{a}
      \KwIn{batch of labeled data $\mathcal{X} =\{(x_{b},y_{b})\}^{B}_{b=1}$, batch of unlabeled data $\mathcal{U} =\{u_{b}\}^{\mu B}_{b=1}$,  Default Classifier $\mathcal{D} $, Auxiliary Classifier $\mathcal{A}$, maximum number of iterations $M$, augmentation $\alpha$}
      \For{iteration $t=1$ to $M$}{
      
      $\overline{y}_{b}=\textrm{randselect}(\mathcal{Y} \setminus \{y_{b}\}),b\in (1,\dots,B)$\\
 
      \hfill\textcolor{light-gray}{ \tcp{Select complementary label from $\mathcal{Y} $ randomly}}
    
  $\mathcal{L}_{sd}=\frac{1}{B}\sum_{n = 1}^{B}H(y_{n},P_{\mathcal{D}}(y_{c}|x_{w,n}))$\\
  \hfill\textcolor{light-gray}{ \tcp{Compute default supervised loss}}
  
  $\mathcal{L}_{sa}=\frac{1}{B}\sum_{n = 1}^{B}H(\overline{y}_{n},P_{\mathcal{A}}(y_{c}|x_{w,n}))$\\
  \hfill\textcolor{light-gray}{ \tcp{Compute auxiliary supervised loss }}
  \For{iteration $b=1$ to $\mu B$}{
            $u_{w,b}=\alpha_{\textrm{weak}}(u_b)$\textcolor{light-gray}{\hfill \tcp{Apply  weak augmentation to $u_b$}}
            $u_{s,b}=\alpha_{\textrm{strong}}(u_b)$\textcolor{light-gray}{\hfill \tcp{Apply  strong augmentation to $u_b$}}
            $p_{b}=P_{\mathcal{D}}(y_{c}|u_{w,b})$\textcolor{light-gray}{\hfill \tcp{Compute predictions of $\mathcal{D}$ for $u_{w,b}$  }}
            $p_{s,b}=P_{\mathcal{D}}(y_{c}|u_{s,b})$\textcolor{light-gray}{\hfill \tcp{Compute predictions of $\mathcal{D}$ for $u_{s,b}$ }}
            $q_{b}=P_{\mathcal{A}}(y_{c}|u_{w,b})$\textcolor{light-gray}{\hfill \tcp{Compute predictions of $\mathcal{A}$ for $u_{w,b}$  }}
            $q_{s,b}=P_{\mathcal{A}}(y_{c}|u_{s,b})$ \textcolor{light-gray}{\hfill \tcp{Compute predictions of $\mathcal{A}$ for $u_{s,b}$ }}
            $\overline{p}_{b}=\textrm{Norm}(\mathbbm{1}-p_{b})$\\
            $\overline{q}_{b}=\textrm{Norm}(\mathbbm{1}-q_{b})$\\
            
            $\Tilde{p}_{b}=\textrm{Norm}(p_{b}\times\frac{\Psi(\overline{q})}{\Psi(p)})$\\
            \textcolor{light-gray}{ \tcp{Apply distribution alignment reciprocally}}
            $\Tilde{q}_{b}=\textrm{Norm}(q_{b}\times\frac{\Psi(\overline{p})}{\Psi(q)})$\textcolor{light-gray}{\hfill \tcp{ Soft complementary labels for $u_{w,b}$}}
            $\hat{\Tilde{p}}_{b}= \arg \max (\Tilde{p}_{b})$\textcolor{light-gray}{\hfill \tcp{ Hard pseudo-labels for $u_{w,b}$}}

  }
   
  $\mathcal{L}_{cd}=\frac{1}{\mu B}\sum_{n = 1}^{\mu B} H(\hat{\Tilde{p}}_{n},p_{s,n})$\hfill\textcolor{light-gray}{ \tcp{Compute default consistency loss }}
          $\mathcal{L}_{ca}=\frac{1}{\mu B}\sum_{n = 1}^{\mu B}H(\Tilde{q}_{n},q_{s,n})$\textcolor{light-gray}{\hfill \tcp{Compute auxiliary consistency loss}}
         \textbf{return} $\mathcal{L}=\mathcal{L}_{sd}+\lambda _{a}\mathcal{L}_{sa}+\lambda _{cd}\mathcal{L}_{cd}+\lambda _{ca}\mathcal{L}_{ca} $ \textcolor{light-gray}{\hfill \tcp{Optimize total loss $\mathcal{L}$}}
      }
\end{algorithm}
\section{Datasets with Mismatched distributions}
\subsection{Protocol of DARP}
\label{app:darp}
DARP \cite{kim2020distribution} introduces this protocol to build a class-imbalanced dataset. DARP introduces two parameters namely imbalanced ratio $\gamma_{l}$ and $\gamma_{u}$ to  control the class-imbalance of dataset. For the labeled data, the data number of each class $N_{i}$ is scaled by:  $N_{i}=N_{1}\times \gamma_{l}^{-\frac{i-1}{n-1}}$, where $i\in(1,\dots,n)$ and $n$ is the number of classes. Likewise, for the unlabeled data, the data number of each class $M_{i}$ is scaled by:  $M_{i}=M_{1}\times \gamma_{u}^{-\frac{i-1}{n-1}}$. Specially, \textit{``reversed''} in Tab. \ref{table:darp} indicates that the unlabeled data with reversely ordered class distribution is used, \ie, $M_{i}=M_{1}\times \gamma_{u}^{-\frac{n-i}{n-1}}$. $N_{1}=1500$ and $M_{1}=3000$ are applied into CIFAR-10 under DARP's protocol. DARP constructs STL-10 with $N_{1}=450$ and fully use the given unlabeled data in this dataset (\ie, $\sum_{i=1}^{n} M_{i}=100,000$). $\gamma_{u}$ is not set for STL-10 due to the unknown ground-truth of the unlabeled data. DARP claims the labeled and unlabeled data in STL-10 have different distributions, \ie, $\gamma_{l}\neq\gamma_{u}$.

Additionally, we show the results of more baseline methods under DARP's protocol \cite{kim2020distribution} in Tab. \ref{table:darp2} for comparison with our method.
\begin{table}[t]
\centering
  \scriptsize
  \caption{Results of accuracy (\%) under DARP’s protocol. We report more baseline results including MixMatch \cite{berthelot2020remixmatch} and ReMixMatch \cite{berthelot2020remixmatch} for comparison with RDA.  Results of baseline methods  are copied from DAPR~\cite{kim2020distribution}. We abbreviate \textbf{R}eMixMatch and \textbf{M}ixMatch as \textbf{R} and \textbf{M}, respectively. }

  \vskip 0in
      \label{table:darp2}
      \setlength{\tabcolsep}{0.4mm}{
      \begin{tabular}{@{}l|cccc|cc@{}}     
      \toprule
      \multirow{3}{*}{Method} & \multicolumn{4}{c|}{CIFAR-10 ($\gamma_{l}=$100)} &\multicolumn{2}{c}{STL-10 ($\gamma_{l}\neq \gamma_{u}$)} \\  \cmidrule(lr){2-5}  \cmidrule(lr){6-7}
      & $\gamma_{u}=1$    & $\gamma_{u}=50$ & $\gamma_{u}=150$  & $\gamma_{u}=100$ (reversed)  &  $\gamma_{l}=10$ & $\gamma_{l}=20$ \\ \cmidrule(r){1-1} \cmidrule(lr){2-5}     \cmidrule(lr){6-7}
      \textbf{M}ixMatch &                   41.50$\pm$0.76& 64.10$\pm$0.58 & 65.50$\pm$0.64& 47.90$\pm$0.09    &56.30$\pm$0.46 &45.20$\pm$0.19     \\
 
     \textbf{M} w. DARP              & 86.70$\pm$0.80 & 68.30$\pm$0.47 & 66.70$\pm$0.25 & 72.90$\pm$0.24 &67.90$\pm$0.24 &58.30$\pm$0.73        \\ 
      \textbf{R}{eMixMatch}                       & 48.30$\pm$0.14 & 75.10$\pm$0.43 & 72.50$\pm$0.10 & 49.00$\pm$0.55 &67.80$\pm$0.45 &60.10$\pm$1.18            \\
 
     \textbf{R} w. DARP              & 89.70$\pm$0.15 & 77.40$\pm$0.22 & 73.20$\pm$0.11 & \textbf{80.10$\pm$0.11} & 79.40$\pm$0.07 & 70.90$\pm$0.44      \\ \cmidrule(r){1-1} \cmidrule(lr){2-5}      \cmidrule(lr){6-7}
      RDA              & \textbf{93.35$\pm$0.24}            & \textbf{79.77$\pm$0.06}             & \textbf{74.48$\pm$0.24}           & {79.25$\pm$0.52}     & \textbf{87.21$\pm$0.44}             &\textbf{83.21$\pm$0.52}       \\ 
      \bottomrule
      \end{tabular}
      }
  \vskip 0in
\end{table}

\subsection{Imbalanced  $C_{x}$}
\label{app:mcd}
We now show the details on how to construct dataset with imbalanced labeled data (\ie, $C_x$ is imbalanced) while keeping the number of labeled data unchanged. Following CIFAR-LT~\cite{cao2019learning}, we mimic the imbalanced $C_x$  by an exponential function: $N_i = N_0 \times\gamma_x^{-\frac{i-1}{n-1}}, i\in(1,\dots n)$ to generate the number of labeled data for class with index $i$, where $n$ is the number of classes. We use different $N_0$ to investigate different scale of imbalance. With $N_{0}$ we set, $\gamma_x$ is calculated by the constraint $\sum_{i=1}^n N_i = D_x$, where $D_x$ is the number of labels we set. We search for a $\gamma_x$ from small to large in natural numbers, so that the progress of search  can be summarized as the following optimization:
\begin{equation}
\begin{aligned} \label{P}
&\hat{\gamma_x}=\mathop{\arg\min}_{\gamma_x} \quad D_x-\sum_{i=1}^n N_i\\
&\begin{array}{r@{\quad}r@{}l@{\quad}l}
s.t. &D_x-\sum_{i=1}^n N_i >0 \\
\end{array}
\end{aligned}
\end{equation}

With obtained $\gamma_x$, we add missing labels for classes other than the first class (\ie, keep the $N_0$ unchanged) in turn until the condition $\sum_{i=1}^n N_i = D_x$ is met. Here we found that the labels that need to be added are less than $n$, which means we can complete this progress by adding  at most one round in turn.

\begin{table}[t]
  \scriptsize
      \centering
      \caption{Accuracy (\%) in open-set SSL. Both Semi-Aves and Sem-Fungi have not only OOD unlabeled data but also in-distribution unlabeled data within class distribution that mismatches with the labeled data \cite{su2021realistic}. 
Unlike native RDA, we set a confidence-based thresholding to serve as a simple filter for OOD samples. While this goes against our original intention of using only distribution alignment to improve pseudo-labeling, it is a compromise for this open-set scenario. We follow the backbone and hyper-parameters for FixMatch (except for threshold $\tau=0.5$) in \cite{su2021realistic} and train models from scratch. 
  }
      \begin{tabular}{@{}l|cc@{}}     
      \toprule
      \multirow{2}{*}{Method} & Semi-Aves &Semi-Fungi  \\  \cmidrule(lr){2-2}   \cmidrule(lr){3-3} 
          & \multicolumn{1}{c}{Top-1 / Top-5}   & \multicolumn{1}{c}{Top-1 / Top-5} \\ \cmidrule(r){1-1} \cmidrule(lr){2-2}   \cmidrule(lr){3-3} 
      FixMatch     & 19.2 / 42.6     & 25.2 / 50.2    \\
      RDA    & \textbf{21.9} / \textbf{43.7}     & \textbf{28.7} / \textbf{51.2}        \\ 
      \bottomrule
      \end{tabular}

  \label{table:ood}
\end{table}

\begin{table}[t]
    \scriptsize
    \centering
            \caption{Results of accuracy (\%) on CIFAR-10 using full labels with 40\% asymmetric noise.  Results of baseline noisy label learning methods are reported in DivideMix~\cite{li2020dividemix}. }
            \label{table:no}
            \vskip 0.1in
            \begin{tabular}{@{}l|c@{}}      
            \toprule
            \multirow{2}{*}{Method} & {CIFAR-10}    \\ \cmidrule(lr){2-2} 
                                    & {40\% asym noise}   \\ \cmidrule(r){1-1} \cmidrule(lr){2-2} 
            \mbox{P-correction \cite{yi2019probabilistic}}  & {88.5}  \\
            \mbox{Joint-Optim} \cite{tanaka2018joint}  & {88.9} \\
            \mbox{Meta-Learning} \cite{li2019learning} & {89.2}  \\
            \mbox{DivideMix \cite{li2020dividemix}}   & \textbf{93.4}  \\ \cmidrule(r){1-1} \cmidrule(lr){2-2}
            RDA              & {90.5}  \\ 
            \bottomrule
            \end{tabular}
\end{table}

\section{Additional Experiments with Mismatched Distributions}
\label{app:add}
\subsection{Mismatched Distributions with Non-overlapping Classes in the Unlabeled Data}
In addition to the mismatched distributions discussed in Sec. \ref{sec:md}, SSL with non-overlapping classes in the unlabeled data is a more generalized mismatched scenario. As mentioned, this distribution mismatch is known as SSL using \textit{out-of-distribution}  \textbf{(OOD)} samples in the unlabeled data \cite{oliver2018realistic} (also known as \textit{open-set SSL}). To explore the robustness of RDA, we experiment under the same setting as Sec. \ref{sec:id} in \cite{oliver2018realistic} and observe  slight accuracy drops of RDA, except for at 100\%  class mismatch extent (sometimes more than 10\% drop).
This is understandable because SSL with OOD samples is very different from our task addressing the mismatched distributions with the same classes and we learn total unlabeled data without OOD sample filters.
Considering the fine-grained datasets \textit{Semi-Aves} \cite{su2021semi} (200/800 in-distribution/OOD classes) and \textit{Semi-Fungi} \cite{su2021realistic} (200/1194 in-distribution/OOD classes) are also used to mimic the OOD setting  \cite{su2021realistic}, we evaluate our RDA  on them. The class distributions of both datasets are long-tailed and mismatched. 
As shown in Tab. \ref{table:ood}, when suffers from both mismatched distributions (in our paper) and OOD samples, RDA can still outperform our main baseline FixMatch by improving the pseudo-labels with in-distribution classes, although some aligned pseudo-labels may be assigned to OOD samples. In the future, we will extend RDA to handle open-set SSL, \eg, detecting OOD samples from the perspective of distribution.
Furthermore, we provide discussions on the mismatched distributions with completely disjoint classes in $C_{x}$ and $C_{u}$. This scenario is an extreme case to SSL with OOD samples and \textit{few-label transfer} proposed in \cite{li2021improve} is closely related to it. Differently, our paper argues that even in the normal SSL setting where $C_x$ and $C_u$ share the same classes, the mismatched distributions could cause significant degradation of many popular SSL methods. Considering RDA is originally designed  to strategically align distributions of overlapping classes, it could not work with completely disjoint $C_x$ and $C_u$.

\subsection{Learning with Symmetric Noisy Labels}
\label{sec:nll}
This is a novel setting different from the previous mismatched setting. We note that there are some subtle connections between dataset with noise and mismatched distributions dataset. We treat the total data in the dataset with noise as labeled data and also treat them as unlabeled data, \ie, this scenario can be seen as a
process of SSL. Asymmetric noise is designed by mapping ground-truth labels to similar classes. \eg, in CIFAR-10, we generate noisy labels by deer$\rightarrow$ horse, dog$\leftrightarrow$ cat, \etc~Thus, we can regard CIFAR-10 with asymmetric noise as a mismatched dataset, \ie, the existence of asymmetric noise increases the ratio of some classes and decreases the ratio of some classes accordingly. 

We evaluate RDA on CIFAR-10 with 40\% asymmetric noise. Following DivideMix~\cite{li2020dividemix}, the backbone used in experiments is 18-layer PreAct ResNet \cite{he2016identity} and we train the models with the same setting in Sec. \ref{sec:id}. 
Although we do not make a special design for noisy label, RDA  still achieves quite competitive performance compared with the  noisy label learning methods shown in Tab. \ref{table:no}.


\end{document}